\documentclass[11pt]{article}

\usepackage[T1]{fontenc}
\usepackage{lmodern}
\usepackage[a4paper,margin=1in]{geometry}
\usepackage{amsmath,amssymb,amsthm,mathtools}
\usepackage{booktabs,array,multirow,tabularx,longtable}
\usepackage{enumitem}
\usepackage{graphicx}
\usepackage[dvipsnames]{xcolor}
\usepackage[authoryear,round,sort]{natbib}
\usepackage[hidelinks]{hyperref}
\usepackage{tikz}
\usetikzlibrary{arrows.meta,positioning,calc,shapes.geometric,fit,backgrounds}
\usepackage{algorithm}
\usepackage{algpseudocode}
\usepackage{caption}
\usepackage{subcaption}
\usepackage{placeins}
\usepackage{setspace}
\usepackage{titlesec}
\usepackage{microtype}
\usepackage{authblk}
\usepackage{url}

\newcolumntype{L}{>{\raggedright\arraybackslash}X}
\titleformat{\section}{\large\bfseries}{\thesection.}{0.5em}{}
\titleformat{\subsection}{\normalsize\bfseries}{\thesubsection.}{0.5em}{}
\titleformat{\subsubsection}{\normalsize\itshape}{\thesubsubsection.}{0.5em}{}

\theoremstyle{definition}
\newtheorem{definition}{Definition}

\theoremstyle{plain}
\newtheorem{lemma}{Lemma}
\newtheorem{proposition}{Proposition}
\newtheorem{theorem}{Theorem}
\newtheorem{corollary}{Corollary}

\newcommand{\Ent}{\operatorname{Ent}}
\newcommand{\GR}{\operatorname{GR}}
\newcommand{\CGR}{\operatorname{CGR}}
\newcommand{\supp}{\operatorname{supp}}
\newcommand{\Ree}{\operatorname{Re}}
\newcommand{\Leaves}{\operatorname{Leaves}}
\newcommand{\Path}{\mathcal{P}}
\newcommand{\SIB}{\operatorname{SIB}}

\title{\bfseries Relevance-Aware Rule: Structural Deletion of Irrelevant Conditions in Decision Trees}
\author[]{Jung-Sik Hong}
\author[]{Jeongeon Lee}
\author[]{Min Kyu Sim}
\author[]{Sangheum Hwang\thanks{Corresponding author.}}
\affil[]{Department of Data Science, Seoul National University of Science and Technology\newline
\texttt{\{hong; jelee; mksim; shwang\}@seoultech.ac.kr}}
\date{}

\begin{document}
\maketitle

\begin{abstract}
Decision trees generate interpretable if--then rules, yet they contain irrelevant conditions (IRCs). These IRCs arise from the structural mechanism of tree splitting and persist even in modern optimal sparse tree induction algorithms. Existing IRC deletion methods overlook this structural mechanism; therefore, they either preserve the original tree too loosely to remain reliable, or too strictly to achieve meaningful simplification. This study provides theoretical foundations for reliable IRC deletion by establishing theorems and propositions related to the underlying IRC mechanism. The key finding is that a binary split shifts class proportions in opposite directions relative to the parent. Specifically, an increase in the class-1 proportion along one branch necessitates an increase in the class-0 proportion along its sibling, thereby generating a C1-link and a C0-link. Based on this structural fact, we propose a structural IRC deletion framework. Relative to each leaf, links that increase the leaf-class proportion are matched, whereas links that increase the proportion of the opposite leaf-class are mismatched. These mismatched links are flagged as structurally suspicious IRC candidates. Rather than deleting them outright, the framework rigorously diagnoses their relevance by assessing prediction reliability. It selectively deletes conditions that are structurally and empirically irrelevant, while strictly protecting those whose deletion would reduce the rule's reliability. Experimental results confirm that the proposed framework achieves substantial rule simplification without sacrificing the reliability of the original tree.
\end{abstract}

\textbf{Keywords:} decision tree; interpretability; irrelevant condition; relevance-aware rule; pruning

\section{Introduction}
\label{sec:introduction}

Decision trees translate model decisions into explicit root-to-leaf rules.
A terminal leaf defines a rule of the form
\[
C_1 \wedge C_2 \wedge \cdots \wedge C_r \Rightarrow y=c,
\]
where \(C_1,\ldots,C_r\) are the conditions on the root-to-leaf path and \(c\) is the leaf class.
Root-to-leaf rules support local interpretation, but recursive tree induction forces each descendant leaf to inherit every upstream condition.
As a result, the root-to-leaf path rule is often over-specified, accumulating \emph{irrelevant conditions} (IRCs) that do not contribute to its decision. Because these IRCs hinder interpretability, their deletion is necessary to ensure clear and concise explanations.

IRC deletion addresses a distinct problem from \emph{global tree simplification}, which encompasses both tree pruning and optimal tree induction. While tree pruning removes or replaces leaves and subtrees to reduce the overall tree size, IRC deletion removes an internal antecedent from a surviving path. A globally simplified tree may therefore still contain IRCs. In addition, IRCs frequently remain even in optimal sparse tree induction algorithms that generate small-sized trees \citep{izza2022tackling}.

To perform the path-level IRC deletion, existing literature has largely explored two prominent directions. \emph{Empirical rule post-pruning} deletes antecedents using statistical test or certainty criteria and may convert a partitioning tree into an overlapping rule-set classifier that requires conflict resolution \citep{quinlan1987generating,quinlan1993c45}. \emph{Logical path-redundancy testing} instead certifies whether antecedent deletion preserves the decision tree's class consequence \citep{izza2022tackling}. They do not provide a structural explanation for why certain conditions become irrelevant in certain paths of trees.

We propose a structural IRC deletion method that uncovers the underlying mechanics of why a condition of a path becomes irrelevant. Our starting point is an elementary but useful \emph{structural fact}: in a binary split, the parent class proportion is a weighted average of the two child proportions. Hence, if one child increases the class-1 proportion, the other child necessarily increases the class-0 proportion. We call the corresponding outgoing links \emph{class-1-increasing links} (C1-links) and \emph{class-0-increasing links} (C0-links). This $C1/C0$ link annotation serves as the primary structural signal used in this paper. 

The $C1/C0$ annotation itself is merely a local diagnostic signal, not a judgment of relevance. To formulate such a judgment, we evaluate this signal relative to a specific leaf node of class $c$. A link in the path is \emph{matched} when its orientation aligns with the leaf class $c$, and \emph{mismatched} when it points toward the opposite class. This \emph{leaf-relative diagnosis} provides a structural warning, yet it remains inconclusive; a mismatched link is structurally suspicious but not inherently irrelevant.

To rigorously assess whether a condition should be deleted, we define \emph{relevance-aware rules}---rules that preserve conditions essential for class support or reliability refinement, while removing only those that are structurally and empirically redundant. To this end, we distinguish three analytical layers:
\begin{enumerate}[label=(\roman*)]
\item \emph{Local annotation}: where a link is tagged as C1, C0, or neutral;
\item \emph{Leaf-relative diagnosis}: where a link is classified as matched or mismatched;
\item \emph{Path-level effect}: where deleting the condition is evaluated by its impact on the rule.
\end{enumerate}
At the third layer, we identify a \emph{positive-relevance effect} (where a condition preserves path reliability), a \emph{negative-relevance effect} (where a condition identifies a lower-reliability subregion), or a \emph{probabilistic IRC effect} (where the condition is redundant). By explicitly separating these layers, we avoid the false assumption that every mismatched link is inherently irrelevant, ensuring that our final \emph{relevance-aware rules} are both compact and structurally faithful.

Analytical-layer application depends on how each procedure certifies reliable deletion within a rule domain. Structural candidate generation is domain-independent. The acceptance criterion that converts a candidate into an accepted deletion, however, must be specialized by domain to remain reliable.

Two complementary procedures implement this separation. Method~1, the broader \emph{mismatch-guided} procedure, generates candidates from leaf-relative C1/C0 mismatches. Its acceptance criterion is domain-specific: for deterministic rules, a candidate set is accepted only through a joint \emph{hard-implication certificate}, guaranteeing \(R \Rightarrow y=c\); for class-probability rules, mismatches within label-homogeneous ancestor subtrees are preserved by default, while mismatches elsewhere are deleted only after verifying path-level reliability preservation within a two-sided tolerance. Method~2, the narrower \emph{sibling-certified} procedure, certifies deletion directly from same-class sibling-leaf topology: whenever one child of a split is a same-class leaf and the opposite subtree contains a descendant leaf of that class, the link into that subtree is removable regardless of orientation. This certificate requires no domain-specific criterion, and guarantees exact preservation of the fitted tree's predictions, full coverage, and zero opposite-class rule conflict in either domain. Method~1 thus prioritizes broader relevance-aware compression, whereas Method~2 offers conservative simplification with exact source-tree preservation.

The study makes three contributions:
\begin{itemize}[leftmargin=*,itemsep=0.15em,topsep=0.25em]
    \item Find the structural mechanics of IRC in decision trees.
    \item Formalize path-level relevance semantics to diagnose the functional role of path conditions across three distinct analytical layers.
    \item Develop a structural IRC deletion framework that efficiently identifies and removes IRCs, achieving scalability linear in the number of leaf rules.
\end{itemize}

The remainder of the paper is organized as follows. Section~\ref{sec:related-work} positions the study relative to tree pruning, rule post-pruning, and logical path-redundancy testing. Section~\ref{sec:problem-setting} defines the problem and notation. Section~\ref{sec:method} presents the proposed methods, Section~\ref{sec:experiments} reports the experiments, Section~\ref{sec:discussion} discusses, and Section~\ref{sec:conclusion} concludes the paper.

\section{Related Work}
\label{sec:related-work}

Structural IRC deletion differs from global tree simplification, empirical rule post-pruning, and logical path-redundancy testing in both the object simplified and the behavior preserved.

\subsection{Global tree simplification}

Subtree pruning reduces global tree complexity but does not directly shorten a surviving leaf rule. Cost-complexity, reduced-error, and pessimistic pruning replace subtrees with leaves or select a smaller tree according to predictive and complexity criteria \citep{breiman1984classification,bohanec1994trading}. Subtree-pruning methods can reduce the number of nodes and leaves, but every retained leaf still inherits all conditions on its retained root-to-leaf path. A globally smaller tree may therefore still produce a locally over-specified rule.

Optimization-based tree induction reduces dependence on greedy split selection but does not directly address path-internal condition deletion. Optimization-based tree-induction methods optimize tree-level objectives under structural constraints rather than relying only on greedy top-down growth \citep{bertsimas2017optimal,hu2019optimal}. Global tree optimization, however, does not preclude path-internal IRCs \citep{izza2022tackling}. The proposed IRC-deletion methods can therefore post-process both greedy and optimization-based fitted trees.

\subsection{Empirical rule post-pruning}

Rule post-pruning directly shortens antecedents, but its primary objective is the predictive performance of the resulting rule-set classifier. \citet{quinlan1987generating} extracts rules from a decision tree and deletes antecedents using statistical relevance and certainty criteria. C4.5rules further combines antecedent deletion with estimated-error evaluation, subset selection, and rule ordering \citep{quinlan1993c45}. Related rule-learning methods, including PRISM, RIPPER, and PART, construct or select rule sets outside the original tree partition through modular or separate-and-conquer procedures \citep{cendrowska1987prism,cohen1995fast,frank1998generating}. Independent rule simplification may create overlapping rules with different class predictions. The resulting classifier behavior therefore depends on rule ordering, conflict resolution, coverage, and fallback policies.

Empirical antecedent deletion can improve aggregate rule quality while changing the population represented by an individual leaf rule. Removing an antecedent expands the descendant rule beyond the original leaf region, and the expanded rule may cover cases from sibling or ancestor regions with different structural meanings. A higher empirical class probability can therefore reflect broader coverage rather than faithful simplification of the original leaf. Section~\ref{sec:class-probability-relevance} defines expansion that absorbs a descendant rule into a broader ancestor region as \emph{ancestor-rule assimilation} and evaluates ancestor-rule assimilation separately from predictive accuracy.

Rule compactness alone does not establish explanation quality. Work on comprehensible and interpretable rule models emphasizes that a short rule may omit conditions that identify scientifically or operationally meaningful subgroups \citep{freitas2014comprehensible,rudin2022interpretable}. For decision-tree path rules, a condition may be unnecessary for preserving the final class label but still distinguish leaf regions with different class reliability. The distinction between label preservation and reliability preservation motivates evaluating the resulting rule rather than judging deletion only by antecedent count or one-sided reliability improvement.

\subsection{Logical path-redundancy testing}

Logical path-redundancy methods provide exact preservation guarantees for individual path decisions. \citet{izza2022tackling} test whether deleting a path condition preserves the class consequence under the fitted tree. The path condition is the unit of analysis, and the acceptance criterion prohibits changes to the fitted tree's class decision. Exact fitted-tree preservation is stronger than an empirical tolerance criterion but may be more conservative when limited behavioral change is acceptable. Logical path-redundancy testing certifies deletion safety after tree fitting but does not identify structurally suspicious conditions before evaluation.

\subsection{Position of Relevance-Aware Structural IRC Deletion}

Our structural IRC deletion procedure closes the methodological gap that global tree simplification, empirical rule post-pruning, and logical path-redundancy testing leave open between aggressive antecedent deletion and exact fitted-tree preservation. It does so by using split-induced tree structure to identify suspicious inherited conditions and by associating each accepted deletion with an explicit preservation target. The resulting procedure removes path conditions from a fixed fitted tree while retaining one output rule for every source leaf.

The proposed framework uses two complementary forms of structural evidence. Method 1 uses leaf-relative C1/C0 mismatch together with setting-specific deterministic or reliability-based acceptance criteria, whereas Method 2 uses same-class sibling-leaf topology as a sample-independent certificate for exact source-tree preservation.

\section{Problem Setting and Notation}
\label{sec:problem-setting}

Relevance-aware structural IRC deletion requires explicit definitions of tree links, rule regions, deletion sets, and training-rule reliability.

Let
\[
\mathcal S=\{(x_i,y_i)\}_{i=1}^{M}
\]
be the training sample, where \(M\) is the number of training instances and \(y_i\in\{0,1\}\), and let \(T\) be the fitted binary source tree.
For node \(N\), \(n(N)\) is its training count, \(n^c(N)\) is its class-\(c\) count, \(p_c(N)=n^c(N)/n(N)\) is its class-\(c\) proportion, and \(\Ree(N)\) is its predicted class.

Table~\ref{tab:notation} summarizes the notation; quantities indexed by \(\mathcal S\) are computed on the training sample.

\begin{table}[h!]
\centering
\caption{Core notation for the tree, rule regions, structural diagnostics, and relevance-aware structural IRC deletion.}
\label{tab:notation}
\footnotesize
\begin{tabularx}{\textwidth}{@{}>{\raggedright\arraybackslash}p{0.24\textwidth}L@{}}
\toprule
Notation & Meaning\\
\midrule
\(\mathcal S=\{(x_i,y_i)\}_{i=1}^{M}\) & Training sample of \(M\) labeled instances.\\
\(N_p,N_c\), \(e=(N_p,N_c)\) & Parent, child, and their directed tree link.\\
\(C(e)\) & Path condition associated with link \(e\).\\
\(n(N),n^c(N),p_c(N)\) & Node count, class-\(c\) count, and class-\(c\) proportion.\\
\(\Ree(N)\) & Predicted class of node \(N\).\\
\(\Path(L)\), \(\GR(L)\) & Root-to-leaf link sequence and generated rule for leaf \(L\).\\
\(R^{-Q}\) & Rule obtained from \(R\) by deleting links in \(Q\).\\
\(\mathcal D_{\mathcal S}(R)\) & Indices of training cases satisfying rule \(R\).\\
\(\supp_{\mathcal S}(R)\) & Rule support \(|\mathcal D_{\mathcal S}(R)|\).\\
\(\hat p_{c,\mathcal S}(R)\) & Rule reliability: class-\(c\) fraction among cases satisfying \(R\).\\
\(D_c(Q;R)\) & Change in rule reliability after deleting set \(Q\).\\
\(\epsilon\) & Operational tolerance for the set-level training-reliability change in M1-P.\\
\(\texttt{rf\_tol}\) & Numerical tolerance for C1/C0 orientation ties.\\
\(\Leaves(N),\Lambda(N)\) & Descendant leaves and their set of predicted labels below \(N\).\\
\(\SIB(L)\) & Sibling-leaf-certified candidate links for leaf \(L\).\\
\(H_c(L)\) & Largest class-\(c\) label-homogeneous ancestor subtree containing \(L\).\\
\(I,J,Q\) & Ordered candidate list, accepted deletion set, and trial deletion set.\\
\(\CGR_{1,D},\CGR_{1,P}\) & Deterministic and class-probability outputs of the mismatch-guided mechanism.\\
\(\CGR_{2,D},\CGR_{2,P}\) & Deterministic and class-probability outputs of the sibling-certified mechanism.\\
\bottomrule
\end{tabularx}
\end{table}
\FloatBarrier

Each link \(e\) has an associated path condition \(C(e)\).
Let \(\Path(L)=(e_1,\ldots,e_r)\) denote the ordered root-to-leaf link sequence for leaf \(L\).
The generated rule is
\[
\GR(L)=\bigwedge_{j=1}^{r}C(e_j).
\]
For a deletion set \(Q\subseteq\{e_1,\ldots,e_r\}\), define
\[
R^{-Q}=\bigwedge_{e_j\notin Q}C(e_j).
\]
For a singleton, \(R^{-e}=R^{-\{e\}}\).
The sequence order is used when candidates are evaluated; \(Q\) and the accepted deletion set are sets of links.
Deleting a link means deleting its associated path condition; links and conditions remain distinct objects.
Each structural deletion variant produces exactly one output rule per input leaf and performs neither cross-rule deduplication nor redundant-rule elimination.

The training-sample rule region, support, and rule reliability are
\[
\mathcal D_{\mathcal S}(R)=\{i:x_i\models R\},\qquad
\supp_{\mathcal S}(R)=|\mathcal D_{\mathcal S}(R)|,
\]
\[
\hat p_{c,\mathcal S}(R)=
\frac{|\{i\in\mathcal D_{\mathcal S}(R):y_i=c\}|}
{\supp_{\mathcal S}(R)}.
\]
We call \(\hat p_{c,\mathcal S}(R)\) \emph{rule reliability} throughout and omit \(\mathcal S\) when the training sample is unambiguous.

\begin{definition}[Deterministic rule and class-probability rule settings]
In a deterministic rule setting, a generated rule is a hard implication \(R\Rightarrow y=c\), and \(Q\) consists of IRCs when \(R^{-Q}\Rightarrow y=c\) remains valid.
In a class-probability rule setting,
\[
R\Rightarrow\bigl(c,\hat p_c(R)\bigr),
\]
so preserving only the predicted label need not preserve the rule's empirical meaning.
\end{definition}

\subsection{C1/C0 link orientation}

\begin{definition}[C1 and C0 links]
\label{def:orientation}
For \(e=(N_p,N_c)\), let
\[
\Delta_1(e)=p_1(N_c)-p_1(N_p).
\]
The link is a C1 link if \(\Delta_1(e)>0\) and a C0 link if \(\Delta_1(e)<0\). Values with \(|\Delta_1(e)|\le\texttt{rf\_tol}\) are treated as neutral and do not supply mismatch candidates.
\end{definition}

\begin{lemma}[Class-ratio complementarity]
\label{lem:ratio}
Let parent \(N_p\) be split into two nonempty children \(N_l,N_r\).
If \(p_1(N_l)>p_1(N_p)\), then \(p_1(N_r)<p_1(N_p)\), equivalently \(p_0(N_r)>p_0(N_p)\).
The symmetric statement holds with \(l,r\) interchanged.
\end{lemma}

\begin{proof}
Let \(w_l=n(N_l)/n(N_p)\) and \(w_r=n(N_r)/n(N_p)\).
Then \(w_l,w_r>0\), \(w_l+w_r=1\), and
\[
p_1(N_p)=w_lp_1(N_l)+w_rp_1(N_r).
\]
If \(p_1(N_l)>p_1(N_p)\), then
\[
w_r\{p_1(N_p)-p_1(N_r)\}=w_l\{p_1(N_l)-p_1(N_p)\}>0,
\]
so \(p_1(N_r)<p_1(N_p)\).
The class-0 claim follows from \(p_0=1-p_1\).
\end{proof}

Lemma~\ref{lem:ratio} justifies complementary C1/C0 link orientation for proportion-changing binary splits.
Figure~\ref{fig:c1c0} uses both line style and text labels, so its two orientations do not depend on color.

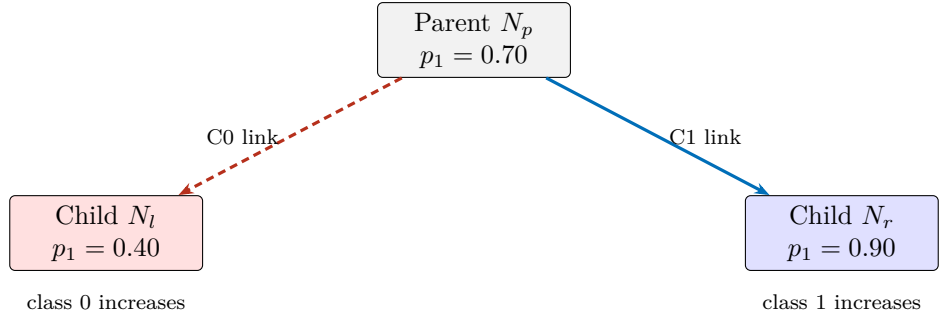
\begin{figure}[h!]
\centering
\begin{tikzpicture}[>=Stealth,node distance=1.55cm and 2.3cm,
nodebox/.style={rectangle,rounded corners=2pt,draw,minimum width=2.55cm,minimum height=0.9cm,font=\small,align=center}]
\node[nodebox,fill=gray!10] (p) {Parent \(N_p\)\\\(p_1=0.70\)};
\node[nodebox,fill=red!12,below left=of p] (l) {Child \(N_l\)\\\(p_1=0.40\)};
\node[nodebox,fill=blue!12,below right=of p] (r) {Child \(N_r\)\\\(p_1=0.90\)};
\draw[BrickRed,very thick,densely dashed,-{Stealth[length=2mm]}] (p)--node[left,font=\scriptsize,black] {C0 link}(l);
\draw[NavyBlue,very thick,solid,-{Stealth[length=2mm]}] (p)--node[right,font=\scriptsize,black] {C1 link}(r);
\node[below=0.18cm of l,font=\scriptsize] {class 0 increases};
\node[below=0.18cm of r,font=\scriptsize] {class 1 increases};
\end{tikzpicture}
\caption{C1/C0 link orientation. A proportion increase toward class 1 on one branch entails a movement toward class 0 on the sibling branch. Red and dashed denote C0; blue and solid denote C1. The labels preserve the distinction in grayscale, and the orientation makes no relevance claim for a descendant rule.}
\label{fig:c1c0}
\end{figure}

\begin{definition}[Matched and mismatched status]
\label{def:status}
Let \(L\) predict class \(c\).
A C1 or C0 link in \(\Path(L)\) is \emph{matched} when it is a C\(c\) link and \emph{mismatched} when it is a C\((1-c)\) link.
These are leaf-relative structural diagnostics, not relevance labels.
\end{definition}

\subsection{Entropy is not link-level relevance}

Define the scalar binary-entropy function
\[
h(t)=-t\log t-(1-t)\log(1-t)
\]
and node entropy \(\Ent(N)=h(p_1(N))\).
This notation separates the scalar function from a node quantity.

\begin{lemma}[Entropy asymmetry]
\label{lem:entropy}
Assume \(p=p_1(N_p)>1/2\).
If child \(N^+\) has \(q=p_1(N^+)>p\), then \(h(q)<h(p)\).
Let \(N^-\) be its sibling and \(s=p_0(N^-)\).
Then
\[
\Ent(N^-)<\Ent(N_p)\Longleftrightarrow s>p,
\]
with equality when \(s=p\), and
\[
\Ent(N^-)>\Ent(N_p)\Longleftrightarrow 1-p<s<p.
\]
\end{lemma}

\begin{proof}
The function \(h\) is symmetric and strictly decreasing on \((1/2,1)\), so \(q>p>1/2\) gives \(h(q)<h(p)\).
For the sibling, \(h(p_1(N^-))=h(s)\), whereas \(\Ent(N_p)=h(p)=h(1-p)\).
The cases follow by comparing the distances of \(s\) and \(1-p\) from \(1/2\).
\end{proof}

Table~\ref{tab:entropy} shows the scientific distinction: complementary C1/C0 orientation is fixed by the proportion change, whereas the sibling entropy can be either higher or lower than the parent's.

\begin{table}[h!]
\centering
\caption{Entropy asymmetry for parent counts \((70,30)\), with a C1 child whose class-1 proportion is 0.90.}
\label{tab:entropy}
\small
\begin{tabular*}{\textwidth}{@{\extracolsep{\fill}}cccc@{}}
\toprule
C1-child counts & C0-sibling counts & Sibling \(p_0\) & Sibling entropy vs.\ parent\\
\midrule
\((9,1)\) & \((61,29)\) & 0.322 & higher\\
\((36,4)\) & \((34,26)\) & 0.433 & higher\\
\((54,6)\) & \((16,24)\) & 0.600 & higher\\
\((63,7)\) & \((7,23)\) & 0.767 & lower\\
\bottomrule
\end{tabular*}
\end{table}

Entropy explains split selection but does not define link-level relevance; the path-level deletion effect on a complete generated rule defines condition relevance.

\section{Relevance-Aware Structural IRC Deletion}
\label{sec:method}

Relevance-aware structural IRC deletion separates structural candidate generation from setting-specific deletion acceptance. Strict mismatch provides a broader diagnostic candidate source, whereas same-class sibling topology provides a narrower exact-preservation certificate independent of link orientation. The implementation names the mismatch-guided branch Method~1 and the sibling-certified branch Method~2; the D/P suffix identifies the deterministic or class-probability setting. Table~\ref{tab:method-map} summarizes the candidate sources, acceptance rules, and preservation targets.

\begin{table}[h!]
\centering
\caption{Structural evidence, setting-specific acceptance, and preservation targets.}
\label{tab:method-map}
\footnotesize
\begin{tabularx}{\textwidth}{@{}p{0.11\textwidth}p{0.25\textwidth}p{0.28\textwidth}L@{}}
\toprule
Variant & Candidate source & Acceptance & Guarantee\\
\midrule
M1-D & Strict-mismatch structural candidates & Complete set passes a joint deterministic certificate & Exact hard implication on the certified deterministic domain.\\
M1-P & IRC-presumed strict mismatches outside \(H_c(L)\) & \textsc{ReliabilitySelect} & Final training-rule reliability differs from its original value by at most \(\epsilon\).\\
M2-D & \(\SIB(L)\) only & Delete the complete certified set & Exact source-tree predictions, full coverage, and zero opposite-class rule conflict.\\
M2-P & \(\SIB(L)\) only & Delete the complete certified set & Exact source-tree predictions and accuracy, full coverage, and zero opposite-class rule conflict.\\
\bottomrule
\end{tabularx}
\end{table}

\subsection{Deterministic structural certificates}

In the deterministic rule setting, an accepted deletion must preserve the hard class implication. C1/C0 mismatch is a candidate signal rather than a certificate, so the complete proposed deletion set is evaluated jointly. Same-class sibling topology instead supplies exact local evidence independently of link orientation.

When a finite deterministic domain \(\mathcal X_0\) is explicitly listed, the joint M1-D certificate can be evaluated directly: a deletion set \(Q\) is accepted only if every \(x\in\mathcal X_0\) satisfying \(R^{-Q}\) has class \(c\). This establishes the weakened implication on \(\mathcal X_0\), without extending it to unlisted attribute combinations. When a class-oriented order is available, the following theorem supplies an analytic sufficient certificate on the stated feasible domain.

\begin{definition}[Class-oriented monotonicity]
For class \(c\), let \(\preceq_c\) be a partial order on the feasible domain \(\mathcal X\). A deterministic response \(f\) is class-\(c\) monotone when \(x'\preceq_c x\) and \(f(x')=c\) imply \(f(x)=c\).
\end{definition}

\begin{theorem}[Joint monotone weakening]
\label{thm:joint-monotone}
Let a valid class-\(c\) rule be
\[
R=A\wedge B_Q\Rightarrow f(x)=c,
\]
where \(B_Q\) is the conjunction of all conditions selected for simultaneous deletion. Suppose \(f\) is class-\(c\) monotone and, for every feasible \(x\models A\), there exists a feasible witness \(x'\models A\wedge B_Q\) such that \(x'\preceq_c x\). Then the jointly weakened rule \(A\Rightarrow f(x)=c\) is valid.
\end{theorem}
\begin{proof}
For any feasible \(x\models A\), the premise supplies \(x'\models A\wedge B_Q\). The original rule gives \(f(x')=c\). Because \(x'\preceq_c x\), class-oriented monotonicity gives \(f(x)=c\). The argument applies to the complete deletion set \(Q\), not to its members one at a time.
\end{proof}

\begin{corollary}[Product Boolean domain]
\label{cor:boolean}
On \(\{0,1\}^d\), suppose \(f\) is monotone in known class-favorable coordinate directions. If every condition in \(B_Q\) restricts a coordinate in the class-unfavorable direction and, for every \(x\models A\), the retained conjunction permits all deleted coordinates to be replaced simultaneously by values satisfying \(B_Q\), then \(Q\) can be deleted jointly.
\end{corollary}
\begin{proof}
Replace all coordinates constrained by \(B_Q\) simultaneously while retaining \(A\). The assumed feasibility supplies the witness required by Theorem~\ref{thm:joint-monotone}.
\end{proof}

Repeated predicates on the same feature are intersected before the feasibility test. Failure of the complete-set certificate preserves the original rule. Individually plausible conditions are not accumulated as though singleton certificates implied joint soundness.

Formally, \(e=(N_p,N_b)\in\SIB(L)\) when \(e\in\Path(L)\), \(N_b\) is the internal child of \(N_p\) containing \(L\), and the other child of \(N_p\) is a leaf with predicted class \(\Ree(L)\).

\begin{proposition}[Sibling-leaf IRC]
\label{prop:sibling}
Let internal node \(N_p\) have a leaf child \(N_a\) predicted as class \(c\) and an internal child \(N_b\). For every descendant leaf \(L\in\Leaves(N_b)\) also predicted as class \(c\), the link from \(N_p\) to \(N_b\) is an IRC with respect to the fitted classifier \(T\).
\end{proposition}
\begin{proof}
Write the descendant rule as \(P_L\wedge C(N_p,N_b)\wedge Q_L\Rightarrow c\). After deleting the link condition, a point satisfying the retained conjunction either continues through \(N_b\) to \(L\), which predicts \(c\), or takes the opposite branch to \(N_a\), which also predicts \(c\). Thus the class implication is preserved with respect to \(T\).
\end{proof}

Figure~\ref{fig:sibling} illustrates the same-class sibling certificate for a class-1 descendant.

\begin{figure}[h!]
\centering
\begin{tikzpicture}[>=Stealth,node distance=1.45cm and 2.15cm,
nodebox/.style={rectangle,rounded corners=2pt,draw,minimum width=2.25cm,minimum height=0.82cm,font=\small,align=center}]
\node[nodebox,fill=gray!10] (p) {\(N_p\)};
\node[nodebox,fill=blue!15,below left=of p] (a) {sibling leaf\\class 1};
\node[nodebox,fill=gray!8,below right=of p] (b) {\(N_b\)\\internal child};
\node[nodebox,fill=blue!12,below left=of b] (d1) {\(L_1\) leaf\\class 1};
\node[nodebox,fill=red!10,below right=of b] (d2) {\(L_2\) leaf\\class 0};
\draw[very thick] (p)--node[left,font=\scriptsize] {same-class leaf}(a);
\draw[very thick,densely dashed] (p)--node[right,font=\scriptsize] {candidate}(b);
\draw[very thick] (b)--(d1);
\draw[very thick] (b)--(d2);
\end{tikzpicture}
\caption{Sibling-leaf certificate for class 1. The dashed link is removable for the same-class descendant \(L_1\), independently of its C1/C0 orientation.}
\label{fig:sibling}
\end{figure}
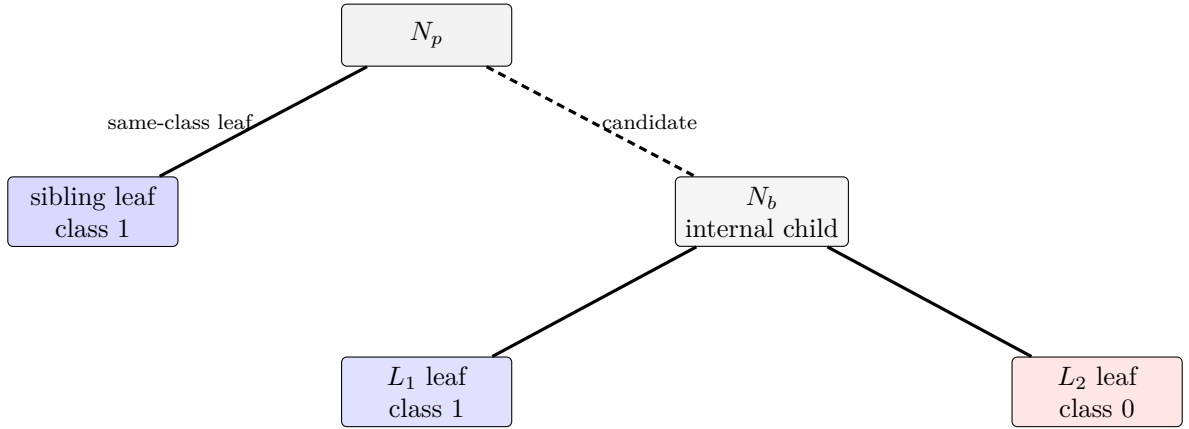

\begin{theorem}[Joint sibling deletion by rootmost divergence]
\label{thm:joint-sibling}
For a leaf \(L\) predicted as class \(c\), let \(Q\subseteq\Path(L)\) contain links whose opposite child is a leaf predicted as \(c\). Deleting every link in \(Q\) simultaneously preserves the class-\(c\) implication of the rule with respect to \(T\).
\end{theorem}
\begin{proof}
Take a point satisfying all retained conditions. If it follows the source path at every deleted link, it reaches \(L\). Otherwise, consider the deleted link closest to the root at which it leaves the source path. The opposite branch terminates immediately at a class-\(c\) sibling leaf. Thus every admitted point is classified as \(c\), regardless of lower deleted links.
\end{proof}

\begin{corollary}[Fitted-tree rule-set preservation]
\label{cor:m2-rule-set}
Apply Theorem~\ref{thm:joint-sibling} to every leaf rule while retaining one output rule per source leaf. For every input \(x\), at least one output rule covers \(x\), and every covering rule predicts \(T(x)\). Hence the output rules preserve fitted-tree predictions and accuracy under any conflict resolver, retain full coverage, and have zero opposite-class rule conflict.
\end{corollary}
\begin{proof}
The generalized rule originating from the source leaf of \(x\) still covers \(x\), establishing coverage. Theorem~\ref{thm:joint-sibling} implies that any other generalized rule covering \(x\) has consequent \(T(x)\), so all covering rules agree.
\end{proof}

The sibling certificate is independent of C1/C0 orientation. Consequently, \(\SIB(L)\) may include a matched link or a link inside \(H_c(L)\); neither property changes sibling-certified candidacy.
\FloatBarrier

\subsection{Class-probability relevance}
\label{sec:class-probability-relevance}

For a class-\(c\) generated rule \(R\) and a deletion set \(Q\), define
\[
D_c(Q;R)=\hat p_c(R^{-Q})-\hat p_c(R).
\]
The single-condition effect is \(D_c(\{e\};R)\), abbreviated \(D_c(e;R)\). The set-level definition covers both joint and cumulative greedy trials.

For tolerance \(\epsilon\ge0\), a deletion set is reliability-decreasing when \(D_c(Q;R)<-\epsilon\), reliability-increasing when \(D_c(Q;R)>\epsilon\), and reliability-preserving when \(|D_c(Q;R)|\le\epsilon\).

\begin{definition}[Path-level relevance effects]
\label{def:path-relevance}
For a single condition \(e\) in a class-\(c\) rule \(R\):
\begin{enumerate}[label=(\roman*),leftmargin=2.0em,itemsep=0.15em,topsep=0.2em]
\item A \emph{positive-relevance effect} occurs when \(D_c(e;R)<-\epsilon\); deleting \(e\) lowers class-\(c\) rule reliability beyond tolerance.
\item A \emph{negative-relevance effect} occurs when \(D_c(e;R)>\epsilon\); deleting \(e\) raises class-\(c\) rule reliability beyond tolerance, although \(e\) may encode a lower-reliability same-label refinement.
\item A \emph{probabilistic IRC effect} occurs when \(|D_c(e;R)|\le\epsilon\); deleting \(e\) preserves class-\(c\) rule reliability within tolerance.
\end{enumerate}
The three effects are relative to \(R\), class \(c\), the training sample, and \(\epsilon\); the effects are not intrinsic labels of individual links. When \(|Q|>1\), only the set-level terms are used because condition effects may interact or offset one another.
\end{definition}

Table~\ref{tab:effects} summarizes the set-level terminology and the corresponding singleton interpretation.

\begin{table}[h!]
\centering
\caption{Set-level deletion effects and their singleton relevance interpretations.}
\label{tab:effects}
\footnotesize
\begin{tabularx}{\textwidth}{@{}>{\raggedright\arraybackslash}p{0.25\textwidth}>{\centering\arraybackslash}p{0.24\textwidth}L@{}}
\toprule
Effect & Definition & Interpretation\\
\midrule
Reliability-decreasing deletion & \(D_c(Q;R)<-\epsilon\) & Joint deletion lowers class-\(c\) training-rule reliability beyond tolerance; for \(Q=\{e\}\), deleting the condition produces a positive-relevance effect.\\
Reliability-increasing deletion & \(D_c(Q;R)>\epsilon\) & Joint deletion raises class-\(c\) training-rule reliability beyond tolerance; for \(Q=\{e\}\), deleting the condition produces a negative-relevance effect.\\
Reliability-preserving deletion & \(|D_c(Q;R)|\le\epsilon\) & The complete deletion set preserves net training-rule reliability within tolerance; for \(Q=\{e\}\), deleting the condition produces a probabilistic IRC effect.\\
\bottomrule
\end{tabularx}
\end{table}

For node \(N\), define
\[
\Lambda(N)=\{\Ree(L'):L'\in\Leaves(N)\}.
\]
For a leaf \(L\) with \(c=\Ree(L)\), let \(S\) be the highest ancestor of \(L\) such that \(\Lambda(S)=\{c\}\). The subtree rooted at \(S\) is denoted \(H_c(L)\) and is the largest \(c\)-label-homogeneous ancestor subtree containing \(L\). If no nonleaf ancestor satisfies the condition, set \(H_c(L)=L\). A link lies inside \(H_c(L)\) when both its parent and child belong to the rooted subtree; otherwise the link lies outside.

The mismatch-guided mechanism preserves matched links by construction; matched status alone does not establish positive relevance. Mismatched links inside \(H_c(L)\) are treated as negative-relevance-presumed and protected because the links may encode same-label reliability refinement. Mismatched links outside \(H_c(L)\) are treated as IRC-presumed structural candidates and become accepted deletions only after the two-sided reliability criterion is satisfied. The three presumptions govern mismatch-guided deletion only; sibling-certified candidacy is independent of C1/C0 orientation and does not use \(H_c(L)\).

\begin{definition}[Relevance-aware rule]
\label{def:relevance-aware-rule}
Within this study, a \emph{relevance-aware rule} is a path rule simplified under explicit setting-specific deletion semantics rather than a globally shortest antecedent set. The mismatch-guided class-probability output retains matched links by construction, protects negative-relevance-presumed mismatches inside \(H_c(L)\), and deletes IRC-presumed candidates outside \(H_c(L)\) only after the two-sided reliability criterion is satisfied. Within the broader relevance-aware framework, a sibling-certified output is identified specifically as an exact source-tree-preserving compact rule because the sibling certificate is independent of \(D_c(Q;R)\).
\end{definition}

\paragraph{Ancestor-rule assimilation.}

Ancestor-rule assimilation occurs when condition deletion reduces a descendant generated rule to a strict root-to-ancestor prefix and thereby expands the rule region:
\[
\mathcal D_{\mathcal S}(R_L)
\subsetneq
\mathcal D_{\mathcal S}(R_A).
\]
The operational diagnostic records a strict-prefix output only when the retained predicates equal a proper root-to-ancestor prefix and the corresponding training region expands strictly. A zero or low strict-prefix rate therefore means that few resulting rules satisfy this exact structural criterion; it does not exclude other forms of overlap with ancestor or sibling regions.

Figure~\ref{fig:assimilation} illustrates the distinction between the descendant rule, the retained ancestor prefix, and the additional sibling-region support.

\begin{figure}[h!]
\centering
\begin{tikzpicture}[>=Stealth,node distance=1.35cm and 1.75cm,
nodebox/.style={rectangle,rounded corners=2pt,draw,minimum width=2.25cm,minimum height=0.82cm,font=\small,align=center}]
\node[nodebox,fill=gray!8] (root) {root-side\\retained prefix};
\node[nodebox,fill=blue!12,below=of root] (m) {\(N_m\)\\\(p_1=0.90\)};
\node[nodebox,fill=blue!7,below left=of m] (r) {\(N_r\)\\\(p_1=0.70\)};
\node[nodebox,fill=blue!20,below right=of m] (n3) {\(N_3\) leaf\\\(p_1=0.95\)};
\node[nodebox,fill=red!12,below left=of r] (n1) {\(N_1\) leaf\\\(p_1=0.40\)};
\node[nodebox,fill=blue!14,below right=of r] (n2) {\(N_2\) leaf\\\(p_1=0.80\)};
\draw[ForestGreen,very thick,double] (root)--node[right,font=\scriptsize,black] {strict prefix path}(m);
\draw[BrickRed,very thick,densely dashed] (m)--node[left,font=\scriptsize,black] {C0}(r);
\draw[NavyBlue,very thick] (m)--node[right,font=\scriptsize,black] {C1}(n3);
\draw[BrickRed,very thick,densely dashed] (r)--(n1);
\draw[NavyBlue,very thick] (r)--(n2);
\node[rectangle,draw,dashed,fit=(m)(r)(n3)(n1)(n2),inner sep=5pt,label={[font=\scriptsize]below:conceptual ancestor region}] {};
\end{tikzpicture}
\caption{Ancestor-rule assimilation. Deleting descendant conditions reduces the \(N_2\) rule to a strict prefix ending at \(N_m\), thereby expanding coverage into the sibling region represented by \(N_3\). The operational diagnostic records only this exact root-to-ancestor collapse.}
\label{fig:assimilation}
\end{figure}
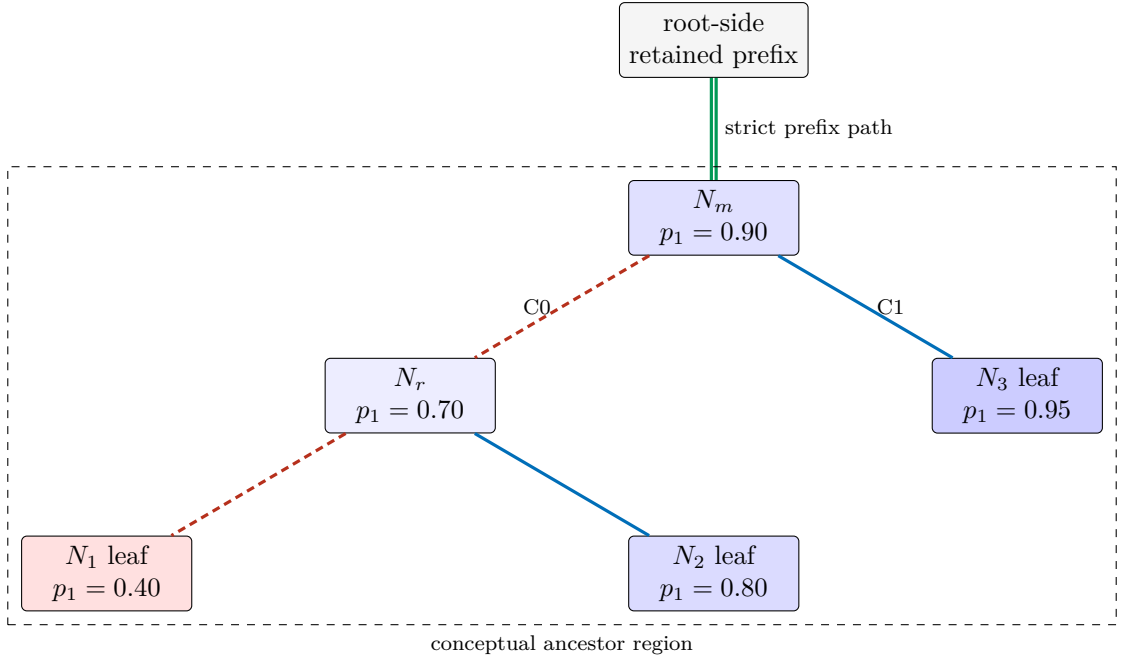

The two-sided criterion controls this set-level reliability change for M1-P. It is not an acceptance condition for M2-P.

\subsection{Executable procedure}

Tree-level annotations and candidate structures are computed once and reused across leaf rules. The mismatch-guided branch uses C1/C0 orientation and, in class-probability mode, \textsc{ReliabilitySelect}. The sibling-certified branch uses only the sibling sets \(\SIB(L)\). Appendix~\ref{app:shared} summarizes the shared preprocessing steps.

\subsubsection{Mismatch-Guided Relevance-Aware Deletion (Method~1)}
\label{sec:method1}

Algorithms~\ref{alg:m1d} and~\ref{alg:m1p} specify the deterministic and class-probability variants of mismatch-guided deletion. M1-D submits the complete strict-mismatch set to either direct finite-domain verification or the joint certificate in Theorem~\ref{thm:joint-monotone}. M1-P treats mismatches inside \(H_c(L)\) as negative-relevance-presumed and protects them; Algorithm~\ref{alg:reliability-select} evaluates the IRC-presumed candidates outside \(H_c(L)\).

\begin{algorithm}[h!]
\caption{M1-D: relevance-aware deterministic deletion from strict mismatches}
\label{alg:m1d}
\begin{algorithmic}[1]
\Require Fitted tree \(T\), deterministic response \(f\), certified domain and joint-certificate routine
\Ensure One rule \(\CGR_{1,D}(L)\) per source leaf
\State Annotate C1/C0 orientation from class-proportion changes.
\For{each leaf \(L\) in source order}
  \State \(R\leftarrow\GR(L)\); let \(I\) be all strict-mismatch structural candidates on \(\Path(L)\), leaf to root.
  \If{the complete set \(I\) satisfies the joint hard-implication certificate}
    \State \(J\leftarrow I\).
  \Else
    \State \(J\leftarrow\emptyset\).
  \EndIf
  \State Output \(\CGR_{1,D}(L)\leftarrow R^{-J}\).
\EndFor
\end{algorithmic}
\end{algorithm}

\begin{algorithm}[h!]
\caption{M1-P: relevance-aware class-probability deletion with two-sided reliability control}
\label{alg:m1p}
\begin{algorithmic}[1]
\Require Fitted tree \(T\), training sample \(\mathcal S\), tolerance \(\epsilon\)
\Ensure One rule \(\CGR_{1,P}(L)\) per source leaf
\State Annotate links and compute \(\Lambda(N)\) once.
\For{each leaf \(L\) in source order}
  \State \(R\leftarrow\GR(L)\), \(c\leftarrow\Ree(L)\); compute \(H_c(L)\).
  \State \(I\leftarrow\) IRC-presumed strict mismatches outside \(H_c(L)\), ordered leaf to root.
  \State \(J\leftarrow\Call{ReliabilitySelect}{R,c,I,\mathcal S,\epsilon}\).
  \State Output \(\CGR_{1,P}(L)\leftarrow R^{-J}\).
\EndFor
\end{algorithmic}
\end{algorithm}

\subsubsection{Sibling-Leaf-Certified Exact-Preservation Deletion (Method~2)}
\label{sec:method2}

Algorithms~\ref{alg:m2d} and~\ref{alg:m2p} specify the deterministic and class-probability variants of sibling-certified deletion. Both variants delete the complete set \(\SIB(L)\) certified by Theorem~\ref{thm:joint-sibling}. The sample-independent structural operation is identical in both settings, and every sibling-certified link is deleted jointly.

\begin{algorithm}[h!]
\caption{M2-D: sibling-certified deterministic IRC deletion}
\label{alg:m2d}
\begin{algorithmic}[1]
\Require Fitted tree \(T\)
\Ensure One rule \(\CGR_{2,D}(L)\) per source leaf
\State Compute \(\SIB(L)\) for every source leaf.
\For{each leaf \(L\) in source order}
  \State \(R\leftarrow\GR(L)\); \(J\leftarrow\SIB(L)\), ordered leaf to root.
  \State Output \(\CGR_{2,D}(L)\leftarrow R^{-J}\) by Theorem~\ref{thm:joint-sibling}.
\EndFor
\end{algorithmic}
\end{algorithm}

\begin{algorithm}[h!]
\caption{M2-P: sibling-certified exact source-tree-preserving deletion}
\label{alg:m2p}
\begin{algorithmic}[1]
\Require Fitted tree \(T\)
\Ensure One rule \(\CGR_{2,P}(L)\) per source leaf
\State Compute \(\SIB(L)\) for every source leaf.
\For{each leaf \(L\) in source order}
  \State \(R\leftarrow\GR(L)\); \(J\leftarrow\SIB(L)\), ordered leaf to root.
  \State Output \(\CGR_{2,P}(L)\leftarrow R^{-J}\) by Theorem~\ref{thm:joint-sibling}.
\EndFor
\end{algorithmic}
\end{algorithm}

\paragraph{Two-sided relevance-aware acceptance for M1-P.}
\textsc{ReliabilitySelect} implements the two-sided relevance-aware acceptance criterion for class-probability mismatch candidates. The subroutine receives the original rule \(R\), class \(c\), ordered candidates \(I\), training sample \(\mathcal S\), and operational tolerance \(\epsilon\), and returns an accepted set \(J\subseteq I\). The helper \textsc{Evaluate} returns the reliability and acceptance decision for a trial set \(Q\). The subroutine evaluates set-level reliability preservation and does not assign intrinsic relevance labels to individual links.

\begin{algorithm}[h!]
\caption{\textsc{ReliabilitySelect}: two-sided relevance-aware acceptance for M1-P}
\label{alg:reliability-select}
\begin{algorithmic}[1]
\Require Rule \(R\), class \(c\), ordered candidates \(I\), \(\mathcal S\), \(\epsilon\), cached condition masks
\Ensure Accepted deletion set \(J\subseteq I\)
\State Compute \(p_{\mathrm{base}}\leftarrow\hat p_{c,\mathcal S}(R)\); initialize trial cache \(\mathcal C_R\).
\Function{Evaluate}{Q}
    \If{\(Q\notin\mathcal C_R\)}
        \State Form the mask for \(R^{-Q}\) by intersecting retained-condition masks.
        \State \(p_Q\leftarrow\hat p_{c,\mathcal S}(R^{-Q})\).
        \State \(a_Q\leftarrow[|p_Q-p_{\mathrm{base}}|\le\epsilon]\).
        \State Cache \((p_Q,a_Q)\) in \(\mathcal C_R\).
    \EndIf
    \State \Return \(\mathcal C_R[Q]\).
\EndFunction
\State \((p_I,a_I)\leftarrow\Call{Evaluate}{I}\).
\If{\(a_I\)}
    \State \Return \(I\).
\EndIf
\State \(J\leftarrow\emptyset\).
\For{each \(e\) in the ordered candidate list \(I\)}
    \State \(Q\leftarrow J\cup\{e\}\).
    \State \((p_Q,a_Q)\leftarrow\Call{Evaluate}{Q}\).
    \If{\(a_Q\)}
        \State \(J\leftarrow Q\).
    \EndIf
\EndFor
\State \Return \(J\).
\end{algorithmic}
\end{algorithm}

The tolerance \(\epsilon\) is operational rather than a sampling-confidence guarantee. Antecedent deletion broadens a nonempty leaf region, so the resulting rule cannot have lower training support than its original leaf rule.
\FloatBarrier

\subsubsection{Procedure Flow Charts}

Figures~\ref{fig:flow-m1} and~\ref{fig:flow-m2} summarize the setting-specific procedures. The mismatch-guided flow sends C1/C0 candidates to either the joint deterministic certificate or the \(H_c(L)\)-protected reliability selector, according to the rule setting. The sibling-certified flow applies the sample-independent sibling certificate in both settings and then assigns the corresponding output notation.

\begin{figure}[h!]
\centering
\resizebox{0.98\linewidth}{!}{%
\begin{tikzpicture}[node distance=0.65cm and 1.0cm,>=Stealth,
box/.style={rectangle,rounded corners=2pt,draw,align=center,font=\scriptsize,minimum width=3.6cm,minimum height=0.55cm},
dia/.style={diamond,draw,align=center,font=\scriptsize,aspect=2.0,inner sep=1pt}]
\node[box,fill=gray!8] (start) {Leaf rule \(\GR(L)\)};
\node[box,below=of start] (ann) {C1/C0 link annotation};
\node[dia,below=0.75cm of ann] (mode) {rule setting?};
\node[box,below left=1.0cm and 2.4cm of mode] (dcan) {M1-D: collect all\\strict mismatches};
\node[box,below right=1.0cm and 2.4cm of mode] (phom) {M1-P: compute \(H_c(L)\)};
\node[dia,below=0.70cm of dcan] (dcert) {complete set passes\\joint certificate?};
\node[box,below=0.65cm of phom] (pcan) {collect IRC-presumed mismatches\\outside \(H_c(L)\)};
\node[box,below left=0.78cm and 0.45cm of dcert] (ddel) {delete complete\\mismatch set};
\node[box,below right=0.78cm and 0.45cm of dcert] (dkeep) {retain source\\rule};
\node[box,below=0.65cm of pcan] (psel) {apply \textsc{ReliabilitySelect}};
\node[box,below=1.65cm of dcert] (dout) {output \(\CGR_{1,D}(L)\)};
\node[box,below=0.88cm of psel] (pout) {output \(\CGR_{1,P}(L)\)};
\draw[->] (start)--(ann);
\draw[->] (ann)--(mode);
\draw[->] (mode)--node[left,font=\tiny]{deterministic}(dcan);
\draw[->] (mode)--node[right,font=\tiny]{class probability}(phom);
\draw[->] (dcan)--(dcert);
\draw[->] (phom)--(pcan);
\draw[->] (pcan)--(psel);
\draw[->] (dcert)--node[left,font=\tiny]{yes}(ddel);
\draw[->] (dcert)--node[right,font=\tiny]{no}(dkeep);
\draw[->] (ddel)--(dout);
\draw[->] (dkeep)--(dout);
\draw[->] (psel)--(pout);
\end{tikzpicture}%
}
\caption{Mismatch-guided procedure (Method~1). C1/C0 annotation supplies structural candidates; M1-D deletes the complete mismatch set only when it passes a joint hard-implication certificate, whereas M1-P protects negative-relevance-presumed mismatches inside \(H_c(L)\) and applies set-level reliability selection outside the protected subtree.}
\label{fig:flow-m1}
\end{figure}
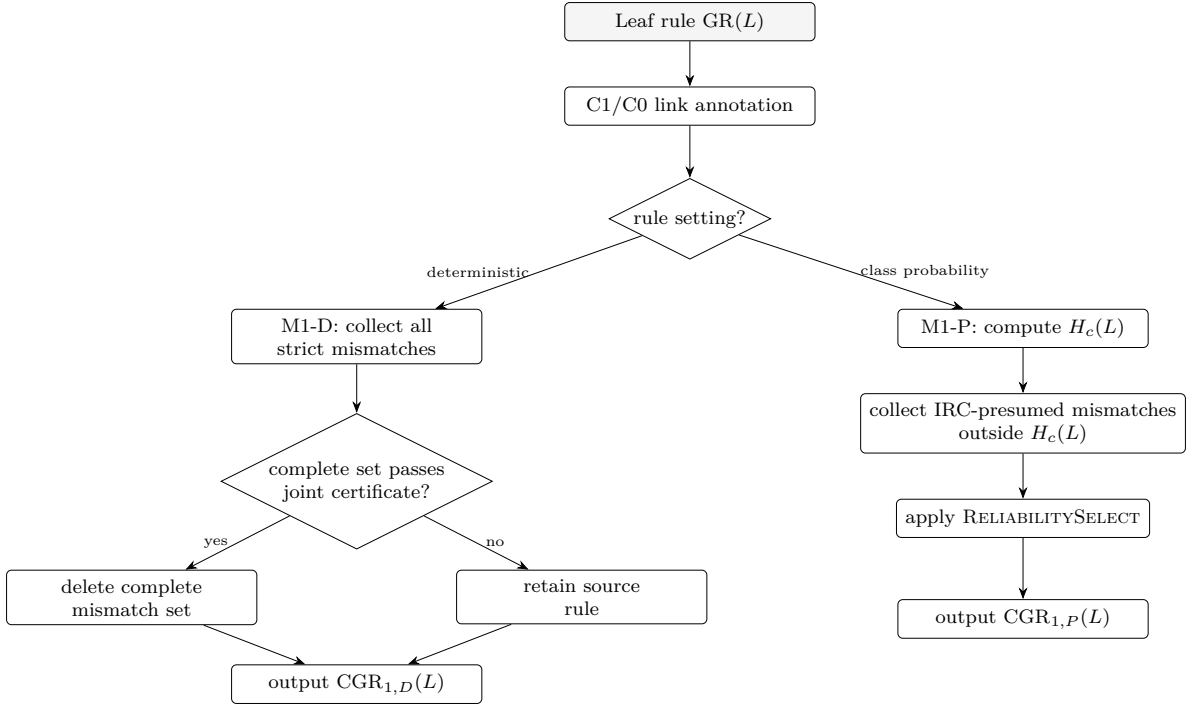

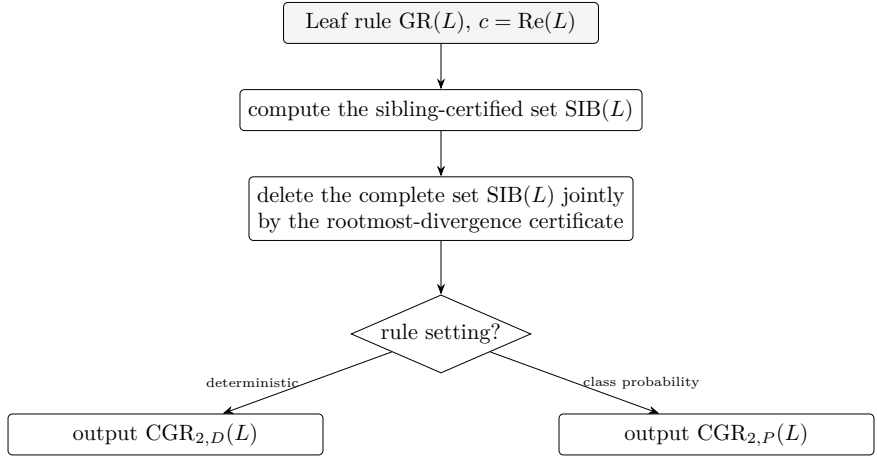
\begin{figure}[h!]
\centering
\resizebox{0.72\linewidth}{!}{%
\begin{tikzpicture}[node distance=0.72cm,>=Stealth,
box/.style={rectangle,rounded corners=2pt,draw,align=center,font=\small,minimum width=5.0cm,minimum height=0.65cm},
dia/.style={diamond,draw,align=center,font=\small,aspect=2.3,inner sep=1pt}]
\node[box,fill=gray!8] (start) {Leaf rule \(\GR(L)\), \(c=\Ree(L)\)};
\node[box,below=of start] (sib) {compute the sibling-certified set \(\SIB(L)\)};
\node[box,below=of sib] (del) {delete the complete set \(\SIB(L)\) jointly\\by the rootmost-divergence certificate};
\node[dia,below=0.85cm of del] (mode) {rule setting?};
\node[box,below left=0.95cm and 1.15cm of mode] (dout) {output \(\CGR_{2,D}(L)\)};
\node[box,below right=0.95cm and 1.15cm of mode] (pout) {output \(\CGR_{2,P}(L)\)};
\draw[->] (start)--(sib);
\draw[->] (sib)--(del);
\draw[->] (del)--(mode);
\draw[->] (mode)--node[left,font=\tiny]{deterministic}(dout);
\draw[->] (mode)--node[right,font=\tiny]{class probability}(pout);
\end{tikzpicture}%
}
\caption{Sibling-certified procedure (Method~2). The complete sibling-certified set is deleted by the same sample-independent exact-preservation transformation before the deterministic or class-probability output is identified.}
\label{fig:flow-m2}
\end{figure}
\FloatBarrier

\subsection{Complexity}
\label{sec:complexity}

Let \(V\) be the number of tree nodes, \(n_L\) the number of leaf rules, \(D\) the maximum path length, \(M\) the number of training instances, and \(k_i\le d_i\le D\) the M1-P candidate count and original length of rule \(i\). Shared preprocessing for link orientation, descendant-label sets, and sibling candidates requires \(O(V+n_LD)\) time.

M1-P additionally constructs cached link masks in \(O(MV)\) time. For rule \(i\), one baseline trial, one joint-deletion trial, and at most \(k_i\) fallback trials each intersect at most \(d_i\) masks over \(M\) instances. Therefore,
\[
T_{\mathrm{M1-P}}
=O\!\left(
V+n_LD+MV+
M\sum_{i=1}^{n_L}d_i(2+k_i)
\right)
\subseteq
O\!\left(V+n_LD+MV+Mn_LD^2\right),
\]
with \(O(MV+n_LD+M)\) auxiliary space.

M2-P uses only tree topology and path entries, giving \(T_{\mathrm{M2-P}}=O(V+n_LD)\) time and \(O(V+n_LD)\) auxiliary space, independent of \(M\). The deterministic variants have corresponding bounds: \(T_{\mathrm{M2-D}}=O(V+n_LD)\), while \(T_{\mathrm{M1-D}}=O(V+n_LD+C_{\mathrm{cert}})\), where \(C_{\mathrm{cert}}\) is the joint-certificate cost. Precomputed monotone feature bounds give \(C_{\mathrm{cert}}=O(n_LD)\); direct verification on a finite table of \(M_0\) rows instead requires \(O(M_0\sum_{i=1}^{n_L}d_i)\subseteq O(M_0n_LD)\).

For the implementations considered here, repeated condition rescoring in the Quinlan procedure \citep{quinlan1987generating} gives \(T_{\mathrm{Quinlan}}=O(M\sum_{i=1}^{n_L}d_i^3)\subseteq O(Mn_LD^3)\), whereas opposite-class path comparisons in the Izza procedure \citep{izza2022tackling} give \(T_{\mathrm{Izza}}=O(n_L^2D^2)\).

One output rule is retained for each source leaf, and the full binary source trees considered here satisfy \(V=2n_L-1\). With \(M\) fixed and lower-order preprocessing terms suppressed, M2 and M1-P scale as \(O(n_LD)\) and \(O(n_LD^2)\), respectively, whereas the Izza and Quinlan procedures scale as \(O(n_L^2D^2)\) and \(O(n_LD^3)\). Thus, the proposed methods retain linear dependence on the number of rules; M2 is linear and M1-P is quadratic in path length.

For fixed \(n_L\) and \(D\), M2 and the Izza procedure are \(O(1)\) in \(M\), whereas M1-P and the Quinlan procedure are \(O(M)\). These bounds isolate the principal sources of growth but do not impose a strict runtime ordering for every combination of \(M\), \(n_L\), and \(D\).

\FloatBarrier

\section{Experiments}
\label{sec:experiments}

The experiments compare the broader mismatch-guided and narrower sibling-certified structural deletion mechanisms with established IRC-deletion procedures. The deterministic variants M1-D and M2-D are first evaluated on decision tables that permit exhaustive case-wise checking. The class-probability evaluation then examines simplification and fidelity: simplification measures how frequently and how strongly source rules are shortened, whereas fidelity measures how closely the resulting rule sets retain the predictive and class-wise behavior of the fitted tree. A runtime study evaluates scalability with respect to the training-sample size and the number of leaf rules.

\subsection{Experimental Setting}

The experiments use two deterministic decision tables and eight class-probability datasets. The Weather data contain 14 cases described by Outlook, Temperature, Humidity, and Windy, with Play as the binary decision \citep{liu2018induction}. The Human Identification (human-id) data contain eight cases described by Height, Hair, and Eyes, with a binary \(+/-\) class \citep{pham1995rules}. The remaining eight binary-classification datasets are drawn from the UCI Machine Learning Repository and OpenML \citep{dua2017uci,openml2013}. Table~\ref{tab:datasets-all} summarizes the inputs used for tree induction.

\begin{table}[h!]
\centering
\caption{Datasets used in the deterministic and class-probability evaluations.}
\label{tab:datasets-all}
\small
\begin{tabular}{llrl}
\toprule
Dataset & Setting & Instances & Input features\\
\midrule
Weather    & Deterministic & 14     & 10\\
human-id   & Deterministic & 8      & 7\\
\addlinespace
adult      & Class-probability & 30,162 & 104\\
backache   & Class-probability & 310    & 12\\
cancer     & Class-probability & 683    & 9\\
EEG        & Class-probability & 14,980 & 14\\
german     & Class-probability & 1,000  & 61\\
heart-h    & Class-probability & 294    & 286\\
ionosphere & Class-probability & 351    & 34\\
spambase   & Class-probability & 4,601  & 57\\
\bottomrule
\end{tabular}
\end{table}

For the deterministic evaluation, all nominal levels, including both Windy values, are represented by full one-hot indicator columns. An unpruned Gini CART with deterministic best-split selection is fitted to each listed decision table and is required to reproduce every supplied class label. M1-D accepts a complete strict-mismatch set only when direct enumeration validates the weakened rule on all supplied cases; M2-D applies the sibling-leaf certificate; and the Izza procedure \citep{izza2022tackling} applies exact fitted-tree path-redundancy testing. The post-deletion audit checks rule validity, cumulative coverage, agreement with source-tree predictions, and opposite-class conflict on every listed case.

For each class-probability dataset, the evaluation reuses 30 stratified 70/30 train--test splits and the corresponding depth-6 CART trees using Gini impurity and saved depth-6 IAI trees representing optimization-based induction \citep{breiman1984classification,bertsimas2017optimal}. M1-P, M2-P, and the procedures of \citet{quinlan1987generating} and \citet{izza2022tackling} receive the same source tree and leaf-rule set for every learner, dataset, and split. For M1-P, \(\epsilon\in\{0,0.01,0.03,0.05,0.10,0.20\}\) is selected separately for each dataset by descriptive grid optimization and shared by CART and IAI\@. The selection maximizes mean condition deletion subject to nonnegative mean accuracy change and a 5\% mean-conflict ceiling for each learner. M2-P uses only the fitted-tree topology.

Where multiple rules cover the same instance, M1-P and the Quinlan procedure rank them by decreasing training reliability, decreasing training support, increasing rule length, and source-leaf order. M2-P and the Izza procedure have zero opposite-class rule conflict, so their predictions are invariant to this ranking. Results are averaged first across the 30 paired splits within each dataset and then equally across the eight dataset means.

\subsection{Deterministic Rule Setting}

Table~\ref{tab:paper0714-deterministic-suite} compares M1-D, M2-D, and the Izza procedure on Weather and human-id. Each source CART fits its listed decision table exactly. The audit enumerates every supplied case and verifies class-valid simplified rules, complete cumulative coverage, unchanged source-tree predictions, and zero opposite-class rule conflict.

% Generated by experiments/paper0714_deterministic_tables.py.
\begin{table}[h!]
\centering
\caption{Three-method validation on the observed deterministic decision tables. RF=1 is the number of simplified rules with empirical reliability one on every covered case. Coverage and conflict are evaluated on all supplied cases. Categorical attributes are represented by full one-hot indicators before fitting the unpruned CART source.}
\label{tab:paper0714-deterministic-suite}
\scriptsize
\setlength{\tabcolsep}{3pt}
\begin{tabular}{llrrrrrrr}
\toprule
Dataset & Method & GR & Removed (\%) & \shortstack{Rules mod.\\(\%)} & $\Delta$length & RF=1 & Coverage & Conflict \\
\midrule
Weather & M1-D & 7 & 30.43 & 71.43 & 1.00 & 7/7 & 100.0 & 0.0 \\
Weather & M2-D & 7 & 21.74 & 57.14 & 0.71 & 7/7 & 100.0 & 0.0 \\
Weather & \citet{izza2022tackling} & 7 & 30.43 & 71.43 & 1.00 & 7/7 & 100.0 & 0.0 \\
\addlinespace
human-id & M1-D & 3 & 20.00 & 33.33 & 0.33 & 3/3 & 100.0 & 0.0 \\
human-id & M2-D & 3 & 20.00 & 33.33 & 0.33 & 3/3 & 100.0 & 0.0 \\
human-id & \citet{izza2022tackling} & 3 & 20.00 & 33.33 & 0.33 & 3/3 & 100.0 & 0.0 \\
\bottomrule
\end{tabular}
\end{table}

M1-D matches the Izza procedure on both deterministic tables and removes more conditions than M2-D on Weather. The Weather source contains 23 conditions across seven rules: M1-D and the Izza procedure each remove seven conditions (30.43\%) from five rules, whereas M2-D removes five conditions (21.74\%) from four rules. Thus, M1-D removes two additional conditions, or 8.69 percentage points more, because finite-domain verification accepts two conditions without a same-class sibling-leaf certificate. On human-id, all three procedures remove the same one of five conditions (20.00\%) from one of three rules. Every simplified rule set preserves complete listed-case coverage, source-tree predictions, and zero opposite-class rule conflict.

\subsection{Class-Probability Rule Setting}

The class-probability evaluation compares the relevance-aware mismatch-guided variant M1-P and the sibling-certified exact-preservation variant M2-P with the Quinlan procedure \citep{quinlan1987generating} and the Izza procedure \citep{izza2022tackling} on the eight datasets in Table~\ref{tab:datasets-all}. Table~\ref{tab:evaluation-metrics} defines the three evaluation dimensions: simplification, fidelity, and scalability.

\begin{table}[h!]
\centering
\caption{Evaluation criteria for the class-probability experiments.}
\label{tab:evaluation-metrics}
\small
\begin{tabularx}{\linewidth}{@{}p{2.2cm}p{3.0cm}X@{}}
\toprule
Criterion & Metric & Interpretation\\
\midrule
Simplification & RuleIRC (\%) & Percentage of source leaf rules from which at least one antecedent condition is deleted.\\
& IRC/rule (\%) & Average within-rule percentage of deleted antecedents, calculated only over rules containing at least one deletion.\\
\addlinespace
Fidelity & \(\Delta\)Accuracy (pp) & Test-accuracy change of the simplified rule set relative to the fitted source tree.\\
& \(|\Delta P_c|,|\Delta R_c|\) (pp) & Absolute class-\(c\) precision and recall changes relative to the fitted source tree.\\
& MacroDev (pp) & Mean absolute deviation of \(P_0,R_0,P_1,R_1\) from their source-tree values.\\
& Rule conflict (\%) & Percentage of test instances covered by rules that predict different classes.\\
\addlinespace
Scalability & Runtime (s) & Simplification time as the number of leaf rules or the sample size increases.\\
\bottomrule
\end{tabularx}
\end{table}

\subsubsection{Simplification}

Table~\ref{tab:paper0714-m1m2-simplification} reports the percentage of source rules containing at least one deleted condition and the average within-rule deletion percentage among affected rules. The relevance-aware mismatch-guided variant M1-P shortens 72.66\% of IAI rules and 70.98\% of CART rules, approximately 2.22 and 2.35 times the corresponding M2-P rates. Relative to the Izza procedure, the M1-P RuleIRC rate is higher by 40.00 percentage points for IAI and 30.63 percentage points for CART\@. Within affected rules, M1-P deletes 34.88\% and 37.13\% of antecedents, exceeding M2-P by 11.63 and 12.97 percentage points and the Izza procedure by 11.13 and 13.16 percentage points. The Quinlan procedure remains the most aggressive, shortening 96.32\%/98.05\% of rules and deleting 63.43\%/64.49\% of antecedents within affected rules. The results place broader reliability-controlled structural deletion between the empirical Quinlan procedure and the source-prediction-preserving M2-P and Izza procedures.

% Generated by experiments/build_paper0714_m1m2.py.
\begin{table}[h!]
\centering
\caption{Rule simplification on the same saved source trees. RuleIRC is the percentage of source rules from which at least one condition is deleted. IRC/rule is the mean percentage of antecedents deleted within those affected rules. Values are dataset-balanced mean $\pm$ SD across eight dataset means.}
\label{tab:paper0714-m1m2-simplification}
\small
\begin{tabular}{llrr}
\toprule
Learner & Method & RuleIRC (\%) & IRC/rule (\%) \\
\midrule
IAI & M1-P & 72.66 $\pm$ 21.57 & 34.88 $\pm$ 8.71 \\
IAI & M2-P & 32.70 $\pm$ 10.99 & 23.25 $\pm$ 3.77 \\
IAI & \citet{quinlan1987generating} & 96.32 $\pm$ 3.19 & 63.43 $\pm$ 14.06 \\
IAI & \citet{izza2022tackling} & 32.66 $\pm$ 10.84 & 23.75 $\pm$ 3.84 \\
\addlinespace
CART & M1-P & 70.98 $\pm$ 17.25 & 37.13 $\pm$ 8.86 \\
CART & M2-P & 30.15 $\pm$ 15.46 & 24.16 $\pm$ 5.94 \\
CART & \citet{quinlan1987generating} & 98.05 $\pm$ 1.84 & 64.49 $\pm$ 11.63 \\
CART & \citet{izza2022tackling} & 40.35 $\pm$ 7.44 & 23.97 $\pm$ 5.47 \\
\bottomrule
\end{tabular}
\end{table}

\subsubsection{Fidelity}

Greater condition deletion does not by itself establish better fidelity. Deleting an antecedent expands the covered rule region and can trade compactness for class-wise deviation or opposite-class overlap. Tables~\ref{tab:paper0714-m1m2-fidelity} and~\ref{tab:paper0714-m1m2-macrodev} quantify this trade-off through accuracy change, rule conflict, absolute class-wise changes, and MacroDev; absolute accuracy is omitted because every comparison uses the same fitted source tree.

% Generated by experiments/build_paper0714_m1m2.py.
\begin{table}[h!]
\centering
\caption{Prediction preservation after rule simplification. $\Delta$Accuracy is simplified-rule-set accuracy minus source-tree accuracy; conflict is the percentage of test instances covered by rules predicting different classes. Values are dataset-balanced mean $\pm$ SD across eight dataset means.}
\label{tab:paper0714-m1m2-fidelity}
\small
\begin{tabular}{llrr}
\toprule
Learner & Method & $\Delta$Accuracy (pp) & Conflict (\%) \\
\midrule
IAI & M1-P & +0.51 $\pm$ 0.44 & 2.82 $\pm$ 1.77 \\
IAI & M2-P & 0.00 $\pm$ 0.00 & 0.00 $\pm$ 0.00 \\
IAI & \citet{quinlan1987generating} & +0.75 $\pm$ 1.33 & 50.59 $\pm$ 41.18 \\
IAI & \citet{izza2022tackling} & 0.00 $\pm$ 0.00 & 0.00 $\pm$ 0.00 \\
\addlinespace
CART & M1-P & +0.63 $\pm$ 0.48 & 2.31 $\pm$ 1.49 \\
CART & M2-P & 0.00 $\pm$ 0.00 & 0.00 $\pm$ 0.00 \\
CART & \citet{quinlan1987generating} & +1.04 $\pm$ 1.28 & 55.44 $\pm$ 37.43 \\
CART & \citet{izza2022tackling} & 0.00 $\pm$ 0.00 & 0.00 $\pm$ 0.00 \\
\bottomrule
\end{tabular}
\end{table}

% Generated by experiments/build_paper0714_m1m2.py.
\begin{table}[h!]
\centering
\caption{Class-wise prediction deviations relative to the fitted source tree. MacroDev is the mean of the four absolute precision and recall deviations. Values are dataset-balanced mean $\pm$ SD across eight dataset means.}
\label{tab:paper0714-m1m2-macrodev}
\scriptsize
\begin{tabular}{llrrrrr}
\toprule
Learner & Method & $|\Delta P_0|$ (pp) & $|\Delta R_0|$ (pp) & $|\Delta P_1|$ (pp) & $|\Delta R_1|$ (pp) & MacroDev (pp) \\
\midrule
IAI & M1-P & 0.78 $\pm$ 0.76 & 0.90 $\pm$ 0.79 & 1.02 $\pm$ 0.64 & 1.20 $\pm$ 1.07 & 0.98 $\pm$ 0.66 \\
IAI & M2-P & 0.00 $\pm$ 0.00 & 0.00 $\pm$ 0.00 & 0.00 $\pm$ 0.00 & 0.00 $\pm$ 0.00 & 0.00 $\pm$ 0.00 \\
IAI & \citet{quinlan1987generating} & 4.22 $\pm$ 3.84 & 6.94 $\pm$ 6.23 & 5.36 $\pm$ 5.25 & 7.59 $\pm$ 10.00 & 6.03 $\pm$ 5.29 \\
IAI & \citet{izza2022tackling} & 0.00 $\pm$ 0.00 & 0.00 $\pm$ 0.00 & 0.00 $\pm$ 0.00 & 0.00 $\pm$ 0.00 & 0.00 $\pm$ 0.00 \\
\addlinespace
CART & M1-P & 0.88 $\pm$ 1.00 & 0.80 $\pm$ 0.63 & 1.02 $\pm$ 0.69 & 1.25 $\pm$ 1.03 & 0.99 $\pm$ 0.66 \\
CART & M2-P & 0.00 $\pm$ 0.00 & 0.00 $\pm$ 0.00 & 0.00 $\pm$ 0.00 & 0.00 $\pm$ 0.00 & 0.00 $\pm$ 0.00 \\
CART & \citet{quinlan1987generating} & 3.55 $\pm$ 3.04 & 5.90 $\pm$ 4.65 & 5.19 $\pm$ 5.23 & 6.82 $\pm$ 8.07 & 5.36 $\pm$ 4.47 \\
CART & \citet{izza2022tackling} & 0.00 $\pm$ 0.00 & 0.00 $\pm$ 0.00 & 0.00 $\pm$ 0.00 & 0.00 $\pm$ 0.00 & 0.00 $\pm$ 0.00 \\
\bottomrule
\end{tabular}
\end{table}

The broader reliability-controlled variant M1-P removes 27.19\% of IAI conditions and 27.30\% of CART conditions, exceeding M2-P by 18.84 and 18.71 percentage points and the Izza procedure by 18.89 and 16.53 percentage points. The broader deletion is accompanied by conflict rates of 2.82\%/2.31\%, MacroDev values of 0.98/0.99 percentage points, and mean accuracy changes of \(+0.51\)/\(+0.63\) percentage points.

The sibling-certified exact-preservation variant M2-P and the Izza procedure produce zero accuracy change, zero class-wise deviation, and zero opposite-class rule conflict. The M2-P outcomes agree with the structural preservation result for the fitted source tree.

The Quinlan procedure removes 62.21\%/63.60\% of all conditions, 35.02/36.30 percentage points more than M1-P, but its dataset-balanced conflict rates reach 50.59\%/55.44\% and its MacroDev reaches 6.03/5.36 percentage points. The positive mean accuracy changes of \(+0.75\)/\(+1.04\) percentage points therefore do not establish preservation of source-tree class-wise or rule-set behavior.

\subsubsection{Scalability}

The two-panel runtime analysis in Figure~\ref{fig:paper0714-m1m2-runtime} examines the scaling dimensions identified in Section~\ref{sec:complexity}. Figure~\ref{fig:runtime-fixed-sample} fixes the training-sample size at \(M=2{,}048\) and varies the leaf-rule target over \(4,8,16,32,64,128\). Figure~\ref{fig:runtime-fixed-tree} holds a 32-leaf fitted tree and source-rule signature fixed and evaluates five training-sample sizes from \(M=4{,}096\) to \(M=1{,}048{,}576\). All four procedures receive the same fitted CART and source rules. Three tree seeds, one warm-up, three timed repetitions, and rotating method order are used; the curves show the median and interquartile range.

\begin{figure}[h!]
\centering
\begin{subfigure}[t]{0.49\textwidth}
\centering
\includegraphics[width=\linewidth]{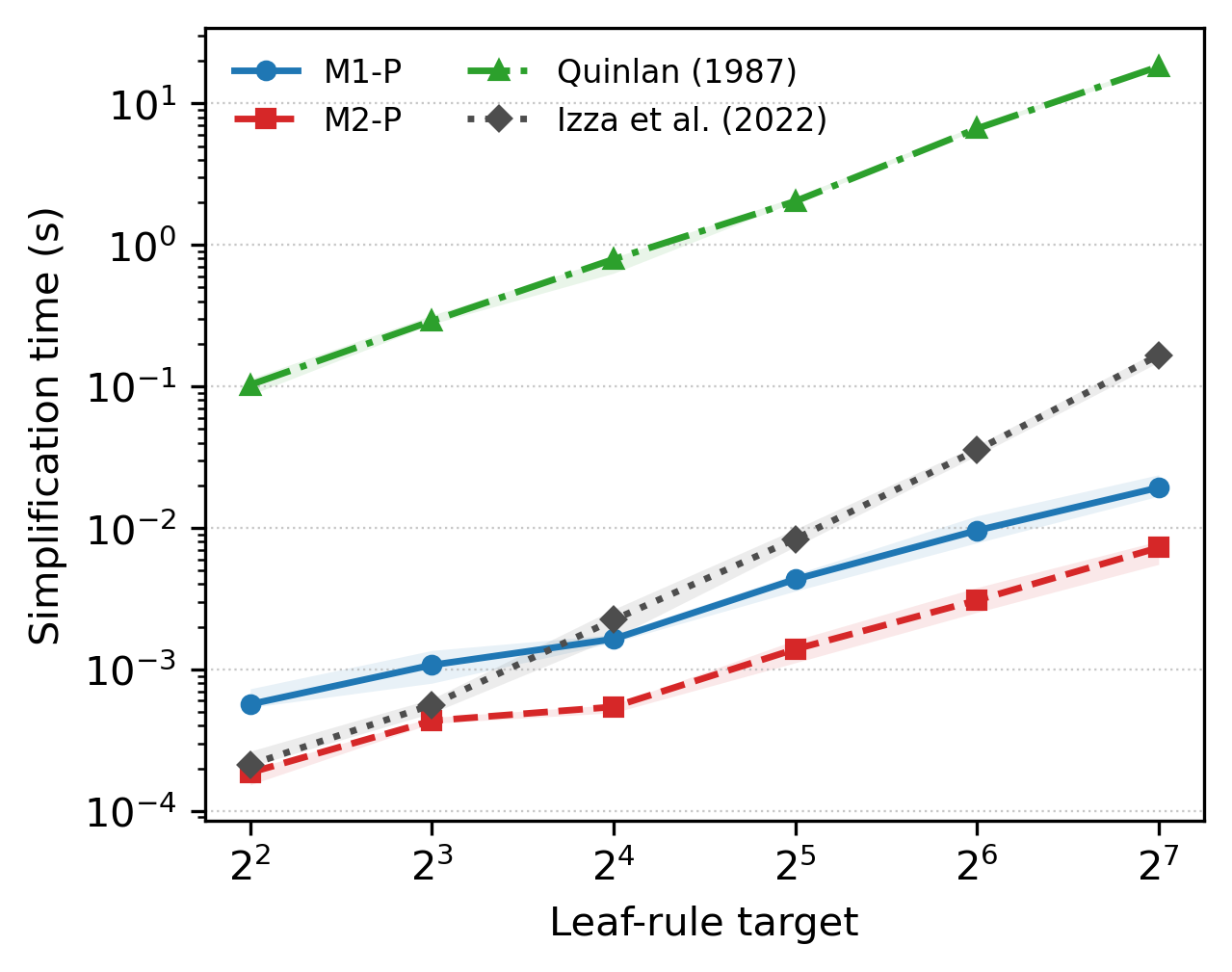}
\caption{Fixed sample size (\(M=2{,}048\))}
\label{fig:runtime-fixed-sample}
\end{subfigure}\hfill
\begin{subfigure}[t]{0.49\textwidth}
\centering
\includegraphics[width=\linewidth]{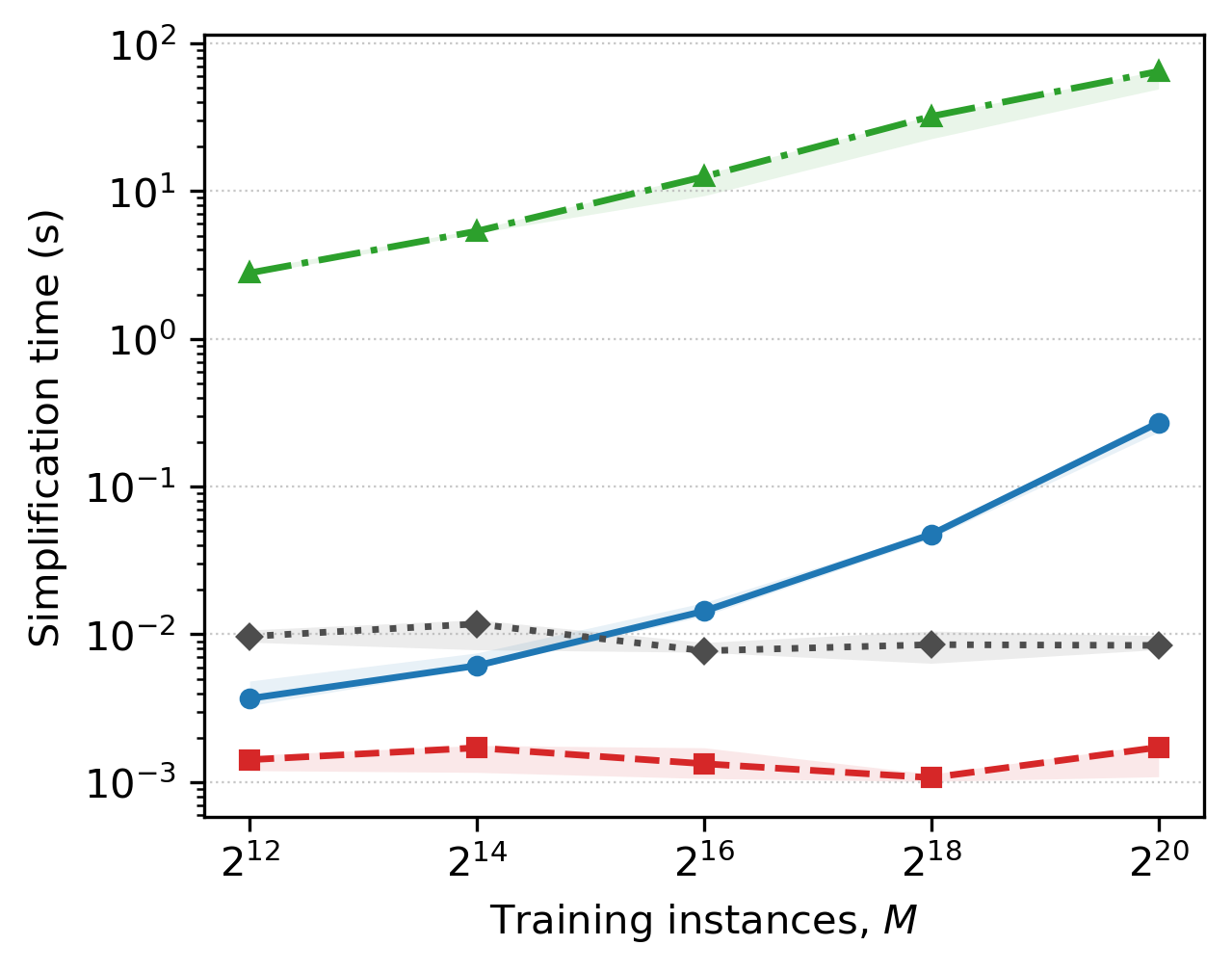}
\caption{Fixed tree size (32 leaves)}
\label{fig:runtime-fixed-tree}
\end{subfigure}
\caption{Simplification-only runtime on identical CART sources for M1-P, M2-P, the Quinlan procedure, and the Izza procedure. Panel (a) varies the leaf-rule target at a fixed training-sample size. Panel (b) fixes one CART per seed and varies the training-sample size through nested subsets that retain every source leaf. Points are medians and bands are interquartile ranges over 3 calls on each of 3 independently generated trees.}
\label{fig:paper0714-m1m2-runtime}
\end{figure}

M2-P and the Izza procedure do not scan the training sample during simplification, and their runtime curves remain nearly flat as \(M\) increases. At 128 leaves, M2-P requires 0.0073~s, compared with 0.166~s for the Izza procedure, making M2-P approximately 23 times faster in this setting. M1-P evaluates instance-level reliability and therefore increases with \(M\), consistently with its linear sample-size term. Cached condition masks nevertheless keep the leaf-scaling cost low: at 128 leaves, M1-P requires 0.0193~s, compared with 18.223~s for the Quinlan procedure, making M1-P approximately 944 times faster in this setting. At \(M=1{,}048{,}576\) with the 32-leaf tree fixed, M1-P requires 0.271~s versus 64.722~s for the Quinlan procedure, and M2-P requires 0.0017~s versus 0.0084~s for the Izza procedure. The corresponding measured speedups are approximately 239-fold and 4.9-fold.

\paragraph{Summary.}
The mismatch-guided class-probability variant M1-P occupies an intermediate simplification regime between the conservative Izza procedure and the aggressive Quinlan procedure. Protection of negative-relevance-presumed mismatches inside \(H_c(L)\) and two-sided reliability selection outside the protected subtree retain low rule conflict and small class-wise deviation while providing substantially faster processing than the Quinlan procedure. The sibling-certified variant M2-P follows a conservative simplification pattern close to exact path-redundancy deletion, exactly preserves the fitted source-tree predictions, and produces zero rule conflict. Within the measured runtime ranges, M2-P is also faster than the Izza procedure, yielding the fastest topology-only simplification among the four procedures.

\FloatBarrier

\section{Discussion}
\label{sec:discussion}

\noindent\textbf{Why are shorter rules not necessarily relevance-aware?}
Condition relevance is a path-level effect of deletion rather than a local property of a split or link. For a single condition \(e\), \(D_c(e;R)<-\epsilon\) denotes a positive-relevance effect because deletion lowers leaf-class reliability; \(D_c(e;R)>\epsilon\) denotes a negative-relevance effect because deletion raises reliability but may erase a lower-reliability same-label refinement; and \(|D_c(e;R)|\le\epsilon\) denotes a probabilistic IRC effect. For multiple conditions, only reliability-decreasing, reliability-increasing, or reliability-preserving set-level terminology applies because individual effects may interact. Rule length and condition deletion must therefore be interpreted together with reliability change, coverage, rule conflict, accuracy, and class-wise deviation.

\noindent\textbf{What does mismatch evidence reveal?}
C1/C0 orientation records the local class-proportion change induced by a split. Relative to a class-\(c\) leaf, a C\(c\) link is matched and a C\((1-c)\) link is mismatched, but neither status alone determines condition relevance. The mismatch-guided mechanism preserves matched links by construction, treats mismatches inside \(H_c(L)\) as negative-relevance-presumed refinements, and treats mismatches outside \(H_c(L)\) as IRC-presumed candidates. Deterministic deletion requires a joint hard-implication certificate, whereas class-probability deletion applies the two-sided effect \(D_c(Q;R)\). The resulting M1-D and M1-P variants therefore provide broader structural deletion under explicit setting-specific control.

\noindent\textbf{What does sibling evidence certify?}
Same-class sibling-leaf topology supplies an exact local certificate independently of C1/C0 orientation and \(D_c(Q;R)\). A sibling-certified link may be matched or located inside \(H_c(L)\), so the sibling and mismatch candidate sets are non-nested. The rootmost-divergence result guarantees source-tree predictions, full coverage, and zero opposite-class rule conflict for the M2-D and M2-P outputs. The narrower sibling-certified set yields less deletion but provides exact source-tree preservation without \(M\) or \(\epsilon\).

\noindent\textbf{What trade-off defines a relevance-aware rule?}
M1-P removes approximately three times as many source conditions as M2-P for both IAI and CART while keeping aggregate MacroDev below one percentage point and rule conflict below 3\%. M2-P provides narrower topology-only deletion with exact preservation, whereas the Quinlan procedure provides substantially greater deletion with conflict on more than half of the test cases and markedly larger class-wise deviation. Relevance-aware rule generation therefore does not maximize deletion but instead makes the structural evidence, acceptance criterion, and preserved rule meaning explicit.

\section{Conclusion}
\label{sec:conclusion}

Condition relevance in a decision-tree path rule is defined by the effect of deletion on the complete rule, not by a local link orientation alone. C1/C0 orientation and matched/mismatched status provide structural diagnostics that generate candidates without declaring relevance. The mismatch-guided mechanism converts structural candidates into accepted deletions only through a joint deterministic certificate or a two-sided class-probability reliability criterion. Same-class sibling topology supplies a separate exact-preservation certificate independent of the path-level reliability effect. A relevance-aware rule is therefore not a globally shortest rule; the deletion rationale and preservation target remain explicit for every accepted condition set.

On Weather, M1-D, M2-D, and the Izza procedure \citep{izza2022tackling} remove 30.43\%, 21.74\%, and 30.43\% of conditions; on human-id, all three remove 20.00\%. Every deterministic rule set preserves listed-case validity, coverage, source-tree predictions, and zero opposite-class rule conflict. On saved CART and IAI trees, M1-P provides substantially greater deletion than M2-P and the Izza procedure while maintaining small accuracy changes, low conflict, and sub-one-percentage-point aggregate MacroDev. M2-P and the Izza procedure preserve source-tree predictions, whereas the Quinlan procedure \citep{quinlan1987generating} achieves the largest reduction at the cost of the greatest rule conflict and class-wise change. The two-axis runtime results further distinguish sample-dependent reliability evaluation from topology-only deletion. Decision-tree rule simplification should be evaluated by which conditions are removed, why deletion is justified, and what rule meaning is preserved.

\section*{Acknowledgments}

This work was supported by the Basic Science Research Program through the National Research Foundation of Korea (NRF), funded by the Ministry of Education (Grant Number: NRF-2021R1A6A1A03039981). This work was also supported by the NRF under Grant RS-2019-NR040071, funded by the Ministry of Science and ICT.

\FloatBarrier
\begingroup
\small
\setlength{\bibsep}{0.25em}
\bibliographystyle{abbrvnat}
\bibliography{references}
\endgroup

\appendix

\section{Shared Preprocessing and Parameters}
\label{app:shared}

The preprocessing stage converts the fitted tree into reusable structural annotations without deleting conditions. Algorithm~\ref{alg:app-preprocess} computes C1/C0 link orientation, descendant-label sets, sibling-leaf candidates, and label-homogeneous ancestor subtrees once before the leaf rules are processed.

\begin{algorithm}[h!]
\caption{Shared preprocessing for relevance-aware structural IRC deletion}
\label{alg:app-preprocess}
\begin{algorithmic}[1]
\Require Fitted binary decision tree \(T\), numerical tolerance \(\texttt{rf\_tol}\)
\Ensure Link orientations, \(\Lambda(N)\), \(\SIB(L)\), and \(H_c(L)\)
\State Initialize \(\SIB(L)\leftarrow\emptyset\) for every leaf \(L\).
\ForAll{nodes \(N\) in postorder}
  \If{\(N\) is a leaf}
    \State \(\Lambda(N)\leftarrow\{\Ree(N)\}\).
  \Else
    \State \(\Lambda(N)\leftarrow\Lambda(N_l)\cup\Lambda(N_r)\).
  \EndIf
\EndFor
\ForAll{links \(e=(N_p,N_c)\)}
  \State \(\Delta_1(e)\leftarrow p_1(N_c)-p_1(N_p)\).
  \If{\(\Delta_1(e)>\texttt{rf\_tol}\)}
    \State Annotate \(e\) as C1.
  \ElsIf{\(\Delta_1(e)<-\texttt{rf\_tol}\)}
    \State Annotate \(e\) as C0.
  \EndIf
\EndFor
\ForAll{internal nodes \(N_p\) with children \(N_a,N_b\)}
  \ForAll{ordered pairs \((N_a,N_b)\), \((N_b,N_a)\)}
    \If{\(N_a\) is a class-\(c\) leaf and \(N_b\) is internal}
      \ForAll{\(L\in\Leaves(N_b)\) with \(\Ree(L)=c\)}
        \State Add link \((N_p,N_b)\) to \(\SIB(L)\).
      \EndFor
    \EndIf
  \EndFor
\EndFor
\ForAll{leaves \(L\) with \(c=\Ree(L)\)}
  \State Order \(\SIB(L)\) from leaf to root; set \(H_c(L)\leftarrow L\).
  \ForAll{nodes \(S\) from \(\operatorname{parent}(L)\) to the root}
    \If{\(\Lambda(S)=\{c\}\)}
      \State \(H_c(L)\leftarrow S\).
    \EndIf
  \EndFor
\EndFor
\end{algorithmic}
\end{algorithm}
\FloatBarrier

The postorder union computes \(\Lambda(N)\), and the ancestor scan obtains the highest label-homogeneous ancestor used by the mismatch-guided class-probability variant M1-P. Sibling detection records the link into an internal child only for same-class descendant leaves when the opposite child is already a class-\(c\) leaf. The sibling-certified mechanism deletes all such links using the resulting topology alone. The numerical \(\texttt{rf\_tol}\) affects C1/C0 diagnostic annotation but cannot change the M2 deletion set.

Table~\ref{tab:parameters} summarizes the experimental and implementation settings used throughout the analysis.

\begin{table}[h!]
\centering
\caption{Common experimental and implementation parameters.}
\label{tab:parameters}
\small
\begin{tabularx}{\textwidth}{@{}p{0.28\textwidth}L@{}}
\toprule
Component & Setting\\
\midrule
Deterministic data & Weather (14 cases, 4 nominal attributes) and human-id (8 cases, 3 nominal attributes); full one-hot encoding.\\
Deterministic source & Exact-fitting CART on all listed cases; Gini criterion; deterministic best split; no depth limit or pruning.\\
M1-D certificate & Direct joint hard-implication check over every listed case in each deterministic table.\\
Class-probability sources & Eight datasets; 30 saved stratified 70/30 splits; depth-6 CART and saved depth-6 IAI trees.\\
Orientation tolerance & \(\texttt{rf\_tol}=10^{-12}\) for numerical ties in C1/C0 annotation.\\
\citet{quinlan1987generating} & Greedy one-condition deletion; empirical certainty factor or one-sided Fisher acceptance; \(\alpha=0.01\), numerical CF tolerance \(10^{-12}\).\\
Candidate order & Leaf to root for M1-P fallback; all sibling-certified links are deleted jointly by M2.\\
Rule production & One output per source leaf; no deduplication or redundant-rule deletion.\\
Timing design & Fixed training-sample size: \(M=2{,}048\), 4--128 leaves; fixed tree size: 32 leaves, \(M=4{,}096\)--\(1{,}048{,}576\); three tree seeds and three timed repetitions.\\
\bottomrule
\end{tabularx}
\end{table}
\FloatBarrier

\section{Detailed Experimental Results}
\label{app:tables}

The following learner-specific tables report the dataset-level results underlying the three class-probability analyses. Their columns match the main-text simplification and fidelity tables: RuleIRC and IRC/rule quantify simplification; \(\Delta\)Accuracy and conflict summarize prediction preservation; and the absolute class-wise changes with MacroDev describe shifts in precision and recall. IAI and CART are presented separately, and each entry averages the 30 paired splits within one dataset.

Tables~\ref{tab:paper0714-m1m2-by-dataset-simplification-iai} and~\ref{tab:paper0714-m1m2-by-dataset-simplification-cart} show that M1-P has a higher RuleIRC rate than M2-P and the Izza procedure on every dataset. Within affected rules, M1-P also deletes a larger antecedent share on every dataset except adult with IAI, where its 25.08\% deletion rate is slightly below the 26.53\% rate of the Izza procedure. The largest M1-P RuleIRC rates occur on spambase for both IAI and CART, at 91.99\% and 89.80\%, whereas EEG yields the smallest rates, at 33.60\% and 45.82\%.

Tables~\ref{tab:paper0714-m1m2-by-dataset-fidelity-iai}--\ref{tab:paper0714-m1m2-by-dataset-class-cart} show the corresponding preservation pattern. M2-P and the Izza procedure retain zero accuracy change, zero rule conflict, and zero class-wise deviation on every dataset. M1-P conflict remains at or below 4.29\%, and its largest MacroDev is 1.77 percentage points for IAI and 1.97 percentage points for CART, both on backache. The Quinlan procedure reaches 100\% conflict on german and heart-h for IAI and on heart-h for CART, while its largest MacroDev occurs on german at 16.81 and 14.36 percentage points. Spambase is a notable exception: the Quinlan and M1-P MacroDev values are similar, although the Quinlan procedure still produces substantially greater rule conflict.

% Generated by experiments/build_paper0714_m1m2.py.
\begingroup
\scriptsize
\clearpage
\setlength{\tabcolsep}{6pt}
\begin{table}[h!]
\centering
\caption{IAI per-dataset simplification outcomes. RuleIRC is the percentage of source rules with at least one deleted condition; IRC/rule is the mean deleted-antecedent percentage within affected rules. Values are means over 30 matched splits.}
\label{tab:paper0714-m1m2-by-dataset-simplification-iai}
\begin{tabular}{llrr}
\toprule
Dataset & Method & RuleIRC (\%) & IRC/rule (\%) \\
\midrule
adult & M1-P & 50.01 & 25.08 \\
adult & M2-P & 33.08 & 22.81 \\
adult & \citet{quinlan1987generating} & 91.54 & 50.98 \\
adult & \citet{izza2022tackling} & 39.40 & 26.53 \\
backache & M1-P & 82.45 & 33.24 \\
backache & M2-P & 37.45 & 21.80 \\
backache & \citet{quinlan1987generating} & 98.49 & 74.43 \\
backache & \citet{izza2022tackling} & 30.87 & 21.56 \\
cancer & M1-P & 89.45 & 37.86 \\
cancer & M2-P & 42.18 & 25.03 \\
cancer & \citet{quinlan1987generating} & 97.78 & 60.32 \\
cancer & \citet{izza2022tackling} & 38.70 & 24.97 \\
EEG & M1-P & 33.60 & 24.44 \\
EEG & M2-P & 19.59 & 19.08 \\
EEG & \citet{quinlan1987generating} & 93.79 & 42.86 \\
EEG & \citet{izza2022tackling} & 18.18 & 19.54 \\
german & M1-P & 61.94 & 27.21 \\
german & M2-P & 17.36 & 18.04 \\
german & \citet{quinlan1987generating} & 99.71 & 77.27 \\
german & \citet{izza2022tackling} & 17.16 & 18.30 \\
heart-h & M1-P & 84.84 & 46.23 \\
heart-h & M2-P & 39.80 & 26.87 \\
heart-h & \citet{quinlan1987generating} & 100.00 & 83.82 \\
heart-h & \citet{izza2022tackling} & 40.68 & 26.70 \\
ionosphere & M1-P & 86.95 & 40.08 \\
ionosphere & M2-P & 47.29 & 29.27 \\
ionosphere & \citet{quinlan1987generating} & 96.04 & 62.46 \\
ionosphere & \citet{izza2022tackling} & 47.16 & 29.45 \\
spambase & M1-P & 91.99 & 44.86 \\
spambase & M2-P & 24.81 & 23.12 \\
spambase & \citet{quinlan1987generating} & 93.18 & 55.27 \\
spambase & \citet{izza2022tackling} & 29.13 & 22.99 \\
\bottomrule
\end{tabular}
\end{table}
\clearpage

\setlength{\tabcolsep}{6pt}
\begin{table}[h!]
\centering
\caption{CART per-dataset simplification outcomes. RuleIRC is the percentage of source rules with at least one deleted condition; IRC/rule is the mean deleted-antecedent percentage within affected rules. Values are means over 30 matched splits.}
\label{tab:paper0714-m1m2-by-dataset-simplification-cart}
\begin{tabular}{llrr}
\toprule
Dataset & Method & RuleIRC (\%) & IRC/rule (\%) \\
\midrule
adult & M1-P & 52.59 & 29.92 \\
adult & M2-P & 11.21 & 20.58 \\
adult & \citet{quinlan1987generating} & 97.15 & 56.12 \\
adult & \citet{izza2022tackling} & 44.41 & 21.56 \\
backache & M1-P & 77.88 & 34.69 \\
backache & M2-P & 32.89 & 22.68 \\
backache & \citet{quinlan1987generating} & 99.31 & 74.36 \\
backache & \citet{izza2022tackling} & 31.31 & 22.26 \\
cancer & M1-P & 86.09 & 36.61 \\
cancer & M2-P & 41.09 & 22.80 \\
cancer & \citet{quinlan1987generating} & 98.30 & 62.27 \\
cancer & \citet{izza2022tackling} & 39.64 & 23.70 \\
EEG & M1-P & 45.82 & 26.52 \\
EEG & M2-P & 12.32 & 17.77 \\
EEG & \citet{quinlan1987generating} & 99.09 & 48.27 \\
EEG & \citet{izza2022tackling} & 33.77 & 18.48 \\
german & M1-P & 54.25 & 29.25 \\
german & M2-P & 22.76 & 19.87 \\
german & \citet{quinlan1987generating} & 99.74 & 73.82 \\
german & \citet{izza2022tackling} & 37.67 & 19.31 \\
heart-h & M1-P & 83.19 & 49.32 \\
heart-h & M2-P & 47.04 & 29.54 \\
heart-h & \citet{quinlan1987generating} & 99.84 & 82.95 \\
heart-h & \citet{izza2022tackling} & 49.05 & 29.76 \\
ionosphere & M1-P & 78.22 & 49.41 \\
ionosphere & M2-P & 51.73 & 36.08 \\
ionosphere & \citet{quinlan1987generating} & 96.20 & 62.57 \\
ionosphere & \citet{izza2022tackling} & 51.91 & 34.56 \\
spambase & M1-P & 89.80 & 41.30 \\
spambase & M2-P & 22.14 & 23.98 \\
spambase & \citet{quinlan1987generating} & 94.78 & 55.60 \\
spambase & \citet{izza2022tackling} & 35.02 & 22.13 \\
\bottomrule
\end{tabular}
\end{table}
\clearpage

\setlength{\tabcolsep}{6pt}
\begin{table}[h!]
\centering
\caption{IAI per-dataset prediction-preservation outcomes. $\Delta$Accuracy is relative to the fitted source tree. Values are means over 30 matched splits.}
\label{tab:paper0714-m1m2-by-dataset-fidelity-iai}
\begin{tabular}{llrr}
\toprule
Dataset & Method & $\Delta$Accuracy (pp) & Conflict (\%) \\
\midrule
adult & M1-P & 0.00 & 0.06 \\
adult & M2-P & 0.00 & 0.00 \\
adult & \citet{quinlan1987generating} & -1.88 & 9.94 \\
adult & \citet{izza2022tackling} & 0.00 & 0.00 \\
backache & M1-P & +0.50 & 4.05 \\
backache & M2-P & 0.00 & 0.00 \\
backache & \citet{quinlan1987generating} & +1.97 & 91.54 \\
backache & \citet{izza2022tackling} & 0.00 & 0.00 \\
cancer & M1-P & +0.85 & 2.85 \\
cancer & M2-P & 0.00 & 0.00 \\
cancer & \citet{quinlan1987generating} & +0.59 & 10.70 \\
cancer & \citet{izza2022tackling} & 0.00 & 0.00 \\
EEG & M1-P & +0.02 & 0.03 \\
EEG & M2-P & 0.00 & 0.00 \\
EEG & \citet{quinlan1987generating} & -0.17 & 23.50 \\
EEG & \citet{izza2022tackling} & 0.00 & 0.00 \\
german & M1-P & +0.34 & 3.60 \\
german & M2-P & 0.00 & 0.00 \\
german & \citet{quinlan1987generating} & +1.00 & 100.00 \\
german & \citet{izza2022tackling} & 0.00 & 0.00 \\
heart-h & M1-P & +0.42 & 4.17 \\
heart-h & M2-P & 0.00 & 0.00 \\
heart-h & \citet{quinlan1987generating} & +2.23 & 100.00 \\
heart-h & \citet{izza2022tackling} & 0.00 & 0.00 \\
ionosphere & M1-P & +1.33 & 4.29 \\
ionosphere & M2-P & 0.00 & 0.00 \\
ionosphere & \citet{quinlan1987generating} & +1.65 & 54.73 \\
ionosphere & \citet{izza2022tackling} & 0.00 & 0.00 \\
spambase & M1-P & +0.62 & 3.54 \\
spambase & M2-P & 0.00 & 0.00 \\
spambase & \citet{quinlan1987generating} & +0.62 & 14.35 \\
spambase & \citet{izza2022tackling} & 0.00 & 0.00 \\
\bottomrule
\end{tabular}
\end{table}
\clearpage

\setlength{\tabcolsep}{6pt}
\begin{table}[h!]
\centering
\caption{CART per-dataset prediction-preservation outcomes. $\Delta$Accuracy is relative to the fitted source tree. Values are means over 30 matched splits.}
\label{tab:paper0714-m1m2-by-dataset-fidelity-cart}
\begin{tabular}{llrr}
\toprule
Dataset & Method & $\Delta$Accuracy (pp) & Conflict (\%) \\
\midrule
adult & M1-P & +0.04 & 0.21 \\
adult & M2-P & 0.00 & 0.00 \\
adult & \citet{quinlan1987generating} & -1.05 & 46.08 \\
adult & \citet{izza2022tackling} & 0.00 & 0.00 \\
backache & M1-P & +1.40 & 4.05 \\
backache & M2-P & 0.00 & 0.00 \\
backache & \citet{quinlan1987generating} & +2.87 & 92.62 \\
backache & \citet{izza2022tackling} & 0.00 & 0.00 \\
cancer & M1-P & +1.02 & 3.48 \\
cancer & M2-P & 0.00 & 0.00 \\
cancer & \citet{quinlan1987generating} & +0.99 & 10.73 \\
cancer & \citet{izza2022tackling} & 0.00 & 0.00 \\
EEG & M1-P & +0.06 & 0.07 \\
EEG & M2-P & 0.00 & 0.00 \\
EEG & \citet{quinlan1987generating} & +0.22 & 52.84 \\
EEG & \citet{izza2022tackling} & 0.00 & 0.00 \\
german & M1-P & +0.29 & 1.84 \\
german & M2-P & 0.00 & 0.00 \\
german & \citet{quinlan1987generating} & +1.37 & 99.58 \\
german & \citet{izza2022tackling} & 0.00 & 0.00 \\
heart-h & M1-P & +0.53 & 2.58 \\
heart-h & M2-P & 0.00 & 0.00 \\
heart-h & \citet{quinlan1987generating} & +2.65 & 100.00 \\
heart-h & \citet{izza2022tackling} & 0.00 & 0.00 \\
ionosphere & M1-P & +0.89 & 3.43 \\
ionosphere & M2-P & 0.00 & 0.00 \\
ionosphere & \citet{quinlan1987generating} & +0.48 & 21.21 \\
ionosphere & \citet{izza2022tackling} & 0.00 & 0.00 \\
spambase & M1-P & +0.79 & 2.79 \\
spambase & M2-P & 0.00 & 0.00 \\
spambase & \citet{quinlan1987generating} & +0.77 & 20.44 \\
spambase & \citet{izza2022tackling} & 0.00 & 0.00 \\
\bottomrule
\end{tabular}
\end{table}
\clearpage

\setlength{\tabcolsep}{2.5pt}
\begin{table}[h!]
\centering
\caption{IAI per-dataset class-wise deviations relative to the fitted source tree. MacroDev is the mean of the four absolute deviations. Values are means over 30 matched splits.}
\label{tab:paper0714-m1m2-by-dataset-class-iai}
\begin{tabular}{llrrrrr}
\toprule
Dataset & Method & $|\Delta P_0|$ (pp) & $|\Delta R_0|$ (pp) & $|\Delta P_1|$ (pp) & $|\Delta R_1|$ (pp) & MacroDev (pp) \\
\midrule
adult & M1-P & 0.05 & 0.05 & 0.01 & 0.02 & 0.03 \\
adult & M2-P & 0.00 & 0.00 & 0.00 & 0.00 & 0.00 \\
adult & \citet{quinlan1987generating} & 10.63 & 16.38 & 3.53 & 2.92 & 8.36 \\
adult & \citet{izza2022tackling} & 0.00 & 0.00 & 0.00 & 0.00 & 0.00 \\
backache & M1-P & 2.14 & 2.56 & 1.07 & 1.32 & 1.77 \\
backache & M2-P & 0.00 & 0.00 & 0.00 & 0.00 & 0.00 \\
backache & \citet{quinlan1987generating} & 9.22 & 9.11 & 3.01 & 6.19 & 6.88 \\
backache & \citet{izza2022tackling} & 0.00 & 0.00 & 0.00 & 0.00 & 0.00 \\
cancer & M1-P & 0.92 & 0.75 & 1.37 & 1.76 & 1.20 \\
cancer & M2-P & 0.00 & 0.00 & 0.00 & 0.00 & 0.00 \\
cancer & \citet{quinlan1987generating} & 1.51 & 1.55 & 2.69 & 2.96 & 2.18 \\
cancer & \citet{izza2022tackling} & 0.00 & 0.00 & 0.00 & 0.00 & 0.00 \\
EEG & M1-P & 0.01 & 0.04 & 0.05 & 0.02 & 0.03 \\
EEG & M2-P & 0.00 & 0.00 & 0.00 & 0.00 & 0.00 \\
EEG & \citet{quinlan1987generating} & 0.79 & 1.10 & 0.81 & 1.77 & 1.12 \\
EEG & \citet{izza2022tackling} & 0.00 & 0.00 & 0.00 & 0.00 & 0.00 \\
german & M1-P & 0.47 & 0.81 & 1.34 & 1.48 & 1.03 \\
german & M2-P & 0.00 & 0.00 & 0.00 & 0.00 & 0.00 \\
german & \citet{quinlan1987generating} & 5.14 & 14.89 & 15.81 & 31.41 & 16.81 \\
german & \citet{izza2022tackling} & 0.00 & 0.00 & 0.00 & 0.00 & 0.00 \\
heart-h & M1-P & 0.50 & 1.13 & 1.45 & 0.83 & 0.98 \\
heart-h & M2-P & 0.00 & 0.00 & 0.00 & 0.00 & 0.00 \\
heart-h & \citet{quinlan1987generating} & 3.47 & 8.45 & 11.06 & 9.48 & 8.11 \\
heart-h & \citet{izza2022tackling} & 0.00 & 0.00 & 0.00 & 0.00 & 0.00 \\
ionosphere & M1-P & 1.66 & 0.70 & 1.23 & 3.33 & 1.73 \\
ionosphere & M2-P & 0.00 & 0.00 & 0.00 & 0.00 & 0.00 \\
ionosphere & \citet{quinlan1987generating} & 2.51 & 2.89 & 4.33 & 5.09 & 3.70 \\
ionosphere & \citet{izza2022tackling} & 0.00 & 0.00 & 0.00 & 0.00 & 0.00 \\
spambase & M1-P & 0.47 & 1.16 & 1.68 & 0.85 & 1.04 \\
spambase & M2-P & 0.00 & 0.00 & 0.00 & 0.00 & 0.00 \\
spambase & \citet{quinlan1987generating} & 0.51 & 1.17 & 1.69 & 0.93 & 1.07 \\
spambase & \citet{izza2022tackling} & 0.00 & 0.00 & 0.00 & 0.00 & 0.00 \\
\bottomrule
\end{tabular}
\end{table}
\clearpage

\setlength{\tabcolsep}{2.5pt}
\begin{table}[h!]
\centering
\caption{CART per-dataset class-wise deviations relative to the fitted source tree. MacroDev is the mean of the four absolute deviations. Values are means over 30 matched splits.}
\label{tab:paper0714-m1m2-by-dataset-class-cart}
\begin{tabular}{llrrrrr}
\toprule
Dataset & Method & $|\Delta P_0|$ (pp) & $|\Delta R_0|$ (pp) & $|\Delta P_1|$ (pp) & $|\Delta R_1|$ (pp) & MacroDev (pp) \\
\midrule
adult & M1-P & 0.15 & 0.08 & 0.02 & 0.04 & 0.07 \\
adult & M2-P & 0.00 & 0.00 & 0.00 & 0.00 & 0.00 \\
adult & \citet{quinlan1987generating} & 6.79 & 10.26 & 2.20 & 2.00 & 5.31 \\
adult & \citet{izza2022tackling} & 0.00 & 0.00 & 0.00 & 0.00 & 0.00 \\
backache & M1-P & 3.04 & 1.89 & 0.95 & 2.01 & 1.97 \\
backache & M2-P & 0.00 & 0.00 & 0.00 & 0.00 & 0.00 \\
backache & \citet{quinlan1987generating} & 9.15 & 7.22 & 3.09 & 5.77 & 6.31 \\
backache & \citet{izza2022tackling} & 0.00 & 0.00 & 0.00 & 0.00 & 0.00 \\
cancer & M1-P & 1.52 & 0.63 & 1.17 & 2.96 & 1.57 \\
cancer & M2-P & 0.00 & 0.00 & 0.00 & 0.00 & 0.00 \\
cancer & \citet{quinlan1987generating} & 1.41 & 1.10 & 2.04 & 2.73 & 1.82 \\
cancer & \citet{izza2022tackling} & 0.00 & 0.00 & 0.00 & 0.00 & 0.00 \\
EEG & M1-P & 0.04 & 0.06 & 0.08 & 0.05 & 0.06 \\
EEG & M2-P & 0.00 & 0.00 & 0.00 & 0.00 & 0.00 \\
EEG & \citet{quinlan1987generating} & 1.43 & 3.23 & 2.26 & 3.97 & 2.72 \\
EEG & \citet{izza2022tackling} & 0.00 & 0.00 & 0.00 & 0.00 & 0.00 \\
german & M1-P & 0.32 & 0.75 & 1.20 & 1.07 & 0.84 \\
german & M2-P & 0.00 & 0.00 & 0.00 & 0.00 & 0.00 \\
german & \citet{quinlan1987generating} & 4.16 & 12.87 & 14.90 & 25.48 & 14.36 \\
german & \citet{izza2022tackling} & 0.00 & 0.00 & 0.00 & 0.00 & 0.00 \\
heart-h & M1-P & 0.48 & 1.13 & 1.58 & 0.94 & 1.03 \\
heart-h & M2-P & 0.00 & 0.00 & 0.00 & 0.00 & 0.00 \\
heart-h & \citet{quinlan1987generating} & 3.32 & 9.52 & 12.24 & 10.21 & 8.82 \\
heart-h & \citet{izza2022tackling} & 0.00 & 0.00 & 0.00 & 0.00 & 0.00 \\
ionosphere & M1-P & 1.06 & 0.50 & 1.10 & 2.11 & 1.19 \\
ionosphere & M2-P & 0.00 & 0.00 & 0.00 & 0.00 & 0.00 \\
ionosphere & \citet{quinlan1987generating} & 1.71 & 1.59 & 2.64 & 3.60 & 2.38 \\
ionosphere & \citet{izza2022tackling} & 0.00 & 0.00 & 0.00 & 0.00 & 0.00 \\
spambase & M1-P & 0.40 & 1.37 & 2.06 & 0.78 & 1.15 \\
spambase & M2-P & 0.00 & 0.00 & 0.00 & 0.00 & 0.00 \\
spambase & \citet{quinlan1987generating} & 0.42 & 1.42 & 2.13 & 0.77 & 1.18 \\
spambase & \citet{izza2022tackling} & 0.00 & 0.00 & 0.00 & 0.00 & 0.00 \\
\bottomrule
\end{tabular}
\end{table}
\endgroup

\FloatBarrier
\end{document}